\documentclass[11pt]{article}
\usepackage[sectionbib]{natbib}
\usepackage{array,epsfig,fancyheadings,rotating}
\usepackage[colorlinks,linkcolor=blue,anchorcolor=blue,citecolor=blue,CJKbookmarks=True]{hyperref}
\usepackage{sectsty, secdot}
\usepackage{amsbsy}
\usepackage{amsmath}
\usepackage{amssymb}
\usepackage{epsfig}
\usepackage{multicol}
\usepackage{multirow}
\usepackage{color}
\usepackage{bm}
\usepackage{threeparttable}
\usepackage{float}
\usepackage{array}
\usepackage{amsmath}
\usepackage{amssymb}
\usepackage{amsfonts}
\usepackage{multirow}
\usepackage{amsthm}
\usepackage{graphicx}
\usepackage{bm}
\usepackage{enumerate}
\usepackage{subfigure}
\usepackage{booktabs}
\usepackage{pdflscape}
\usepackage{url}
\usepackage{xcolor,graphicx,float}
\usepackage{rotating}
\usepackage{times}
\usepackage{indentfirst}
\usepackage{blindtext}
\usepackage{algorithm,algorithmicx,float}
\usepackage{lipsum}
\usepackage[noend]{algpseudocode}
\usepackage{setspace}
\usepackage{booktabs}

\usepackage[sectionbib]{natbib}
\usepackage{array,epsfig,fancyheadings,rotating}
\usepackage[]{hyperref}  %<----modified by Ivan
\usepackage{sectsty, secdot}
\usepackage{authblk}
\usepackage{amsmath}
\usepackage{amssymb}
\usepackage{amsfonts}
\usepackage{multirow}
\usepackage{amsthm}
\usepackage{graphicx}
\usepackage{bm}
\usepackage{enumerate}
\usepackage{subfigure}
\usepackage{booktabs}
\usepackage{pdflscape}
\usepackage{url}
\usepackage{xcolor,graphicx,float}
\usepackage{rotating}
\usepackage{times}
\usepackage{indentfirst}
\usepackage{blindtext}
\usepackage{algorithm,algorithmicx,float}
\usepackage{lipsum}
\usepackage[noend]{algpseudocode}
\usepackage{setspace}
\usepackage{booktabs}


\makeatletter
\renewcommand{\@thesubfigure}{\hskip\subfiglabelskip}
\makeatother
\renewcommand{\baselinestretch} {1.5}
\makeatletter 
\def\singlespace{\def\baselinestretch{1}\@normalsize}

\newlength\savewidth

\@addtoreset{equation}{section}
\newtheorem{theorem}{Theorem}
\newtheorem{assumption}{Assumption}
\newtheorem{lemma}{Lemma}%[section]
\newtheorem{remark}{Remark}%[section]
\newtheorem{corollary}{Corollary}%[section]
\newtheorem{example}{Example}%[section]
\newtheorem{definition}{Definition}%[section]
\usepackage{stackengine}
\newcommand\barbelow[1]{\stackunder[1.2pt]{$#1$}{\rule{.8ex}{.075ex}}}

\DeclareMathOperator*{\argmin}{argmin}

\newcommand{\bg}{\begin{eqnarray}}
\newcommand{\ed}{\end{eqnarray}}
\newcommand{\bgn}{\begin{eqnarray*}}
\newcommand{\edn}{\end{eqnarray*}}

\makeatletter
\renewcommand{\@thesubfigure}{\hskip\subfiglabelskip}
\makeatother
\makeatletter
\def\singlespace{\def\baselinestretch{1}\@normalsize}

\date{\today}

\title{On damage of interpolation to adversarial robustness in regression}

\author{Jingfu Peng}
\author{Yuhong Yang}

\affil{Yau Mathematical Sciences Center, Tsinghua University}

\begin{document}
\begin{sloppypar}

\maketitle

%Abstract
\begin{abstract}

Deep neural networks (DNNs) typically involve a large number of parameters and are trained to achieve zero or near-zero training error. Despite such interpolation, they often exhibit strong generalization performance on unseen data, a phenomenon that has motivated extensive theoretical investigations. Comforting results show that interpolation indeed may not affect the minimax rate of convergence under the squared error loss. In the mean time, DNNs are well known to be highly vulnerable to adversarial perturbations in future inputs. A natural question then arises: Can interpolation also escape from suboptimal performance under a future $X$-attack? In this paper, we investigate the adversarial robustness of interpolating estimators in a framework of nonparametric regression. A finding is that interpolating estimators must be suboptimal even under a subtle future $X$-attack, and achieving perfect fitting can substantially damage their robustness. An interesting phenomenon in the high interpolation regime, which we term the curse of simple size, is also revealed and discussed. Numerical experiments support our theoretical findings.

\end{abstract}
\textbf{KEY WORDS: Adversarial robustness, minimax risk, interpolation, over-parameterization, benign 
overfitting.}

%#####################################################################################################################################
\bigskip
\baselineskip=18pt

\section{Introduction}\label{sec:intro}

Deep neural networks (DNNs) have demonstrated remarkable performance across a wide range of tasks. Despite typically involving a large number of parameters and being trained to achieve zero or near-zero training error \citep[see, e.g.,][]{Gupta2015, Zhang2016understanding, Ma2018, Chaudhari_2019, Du2019, Allen-Zhu2019, Kawaguchi2019}, DNNs often exhibit low generalization error \citep{neyshabur2014search, Zhang2016understanding, Belkin2018To, Belkin2019Reconciling}. This phenomenon appears to contradict the classical statistical perspective that regards interpolation as a proxy for overfitting and, consequently, a source of poor generalization to unseen data. 

Recent years have witnessed a surge in research on the generalization properties of interpolating estimators. In the contexts of linear regression and kernel regression, the phenomena of \emph{double descent} in risk curves and \emph{benign overfitting} of interpolating estimators have been studied \citep[see, e.g.,][]{Belkin2020Two, Bartlett2020Benign, Liang2020Just, Chinot2020robustness, Muthukumar2020, Hastie2022Surprises, Mei2022, Tsigler2023, Lecue2024geometrical, Cheng2024ridge}. In nonparametric regression, a line of research examines whether interpolating estimators can attain consistency and minimax optimal rates in the regular $L_2$-risk \citep{DEVROYE1998209, Belkin2018Overfitting, Belkin2019, Xing2022Benefit, Haas2023, Arnould2023, Chhor2024Benign, Mucke2025}. These findings suggest that interpolation is compatible with minimax optimality under standard risk; in fact, certain kernel ridgeless regression estimators \citep{Haas2023} and local nonparametric rules \citep{Belkin2018Overfitting, Belkin2019, Chhor2024Benign} can simultaneously interpolate the training data while achieving minimax optimality and adaptivity. 

However, while DNNs exhibit impressive performance under normal examples, their lack of robustness to adversarial examples has become a pressing concern. It is well known that subtle, visually imperceptible perturbations to future input data can mislead state-of-the-art deep learning systems into making incorrect predictions \citep{Biggio2013, Szegedy2014, Goodfellow2015, Su2019}. It has also been observed that this lack of robustness is a ubiquitous property among different network architectures and appears in various application domains, including traffic sign recognition \citep{Eykholt2018}, clinical trials \citep{Papangelou2019}, and large language models \citep{zou2023universal}. Given that many of these applications are safety-critical and can directly impact human well-being, safeguarding deep learning models against future adversarial $X$-attacks has become a pressing challenge. Although various \emph{adversarial training} methods have been proposed to enhance robustness \citep[see, e.g.,][]{Goodfellow2015, Madry2018Towards, Wong2018, Raghunathan2018Certified, Cohen2019Certified, Zhang2019, Wang2019Language}, the theoretical mechanisms underlying the fragility of DNNs to future $X$-attacks remain largely elusive. 

The interpolation phenomenon, together with the vulnerability of DNNs under future $X$-attacks, motivates the following fundamental questions: Is interpolation one of the potential sources of the lack of robustness in DNNs? Can a method that interpolates or nearly interpolates the training data still achieve theoretically provable robustness against future $X$-attacks? 

The potential links between interpolation and adversarial robustness has been noted in a limited body of literature. In linear regression and classification tasks, \citet{Donhauser2021} derived an asymptotic expression for the adversarial risk of the \emph{ridgeless} estimator (i.e., the minimum-$\ell_2$-norm interpolator) and showed that ridgeless regression has higher adversarial risk than its regularized counterparts. In a similar setting, \citet{Hao2024surprising} showed that in the benign overfitting regimes where the standard risk of the ridgeless estimator is consistent, its adversarial risk can diverge to infinity. The above results are highly dependent on the specific data distribution and linear model structures. In addition, by focusing on specific interpolators, their theoretical results do not rule out the possibility that other interpolating estimators may achieve superior robust generalization.

\subsection{Contributions}

In this paper, we investigate the impact of interpolation to adversarial robustness in a nonparametric regression setting, where the true regression function belongs to a H\"{o}lder-smooth class. We evaluate the minimax adversarial $L_2$-risk over the class of interpolating estimators, which encompasses all estimators that fit the training data to certain degree. The minimax results obtained in this paper are algorithm-independent, which characterize the optimal adversarial robustness that any interpolator can achieve. This advances the prior work that focuses on the risk consistency of some specific interpolators. 

By comparing the minimax adversarial rate of interpolators with the optimal adversarial rate, we obtain several important conclusions. First, when the interpolation degree is low, the adversarial risk of interpolators matches (in order) the optimal adversarial rate. Second, when the interpolation degree is high (including the exact fitting case), the adversarial risk of interpolators becomes significantly larger than the optimal rates and may fail to converge even under much smaller adversarial perturbations. This demonstrates that interpolation is provably suboptimal and can severely damage the adversarial robustness. Third, we observe an interesting phenomenon referred to as the \emph{curse of sample size}, where increasing the amount of training data deteriorates the adversarial robustness of interpolators.  Simulation results and real data example further support these findings. 

Our minimax risk theory also provides insights into the lack of adversarial robustness in DNNs. We show that any interpolating method, including over-parameterized DNNs that nearly achieve zero training error, cannot attain the optimal robustness. This limitation reflects an inherent deficiency of interpolators and suggests that interpolation may be a key factor underlying the lack of robustness commonly observed in deep learning. To the best of our knowledge, this peculiar fact has not yet been deeply investigated in the rapidly growing literature related to interpolation.

\subsection{Other related work}\label{subsec:related}

\emph{Minimax risk theory in nonparametric regression.} Minimax risk is one of the most widely studied theoretical quantities for benchmarking fundamental statistical hardness and establishing the optimality of procedures \citep[see, e.g.,][]{ibragimov1982bounds, Stone1982, Le_Cambook, birge1986estimating, yu1997assouad, Yang1999Information, Tsybakov2009}. In nonparametric regression, the minimax optimality of neural network estimators has been established \citep[see, e.g.,][]{Schmidt-Hieber2020, kohler2021rate}. Various interpolating estimators have also been shown to attain the minimax optimal rates \citep[see, e.g.,][]{Belkin2018Overfitting, Belkin2019, Haas2023, Chhor2024Benign}. The above works focus on standard risk without adversarial attacks. In the setting where future inputs are subject to adversarial perturbations, the optimal minimax rates have been established in \cite{Peng2024, Peng2025aos}.

\emph{Robust interpolation problem.} \citet{Bubeck2011, Bubeck2023} studied the problem of interpolating training data using two-layer neural networks while keeping the Lipschitz constant of the network bounded. Under certain conditions on the input distribution, they established lower bounds on the Lipschitz constant of any interpolating neural network, demonstrating that over-parameterization is necessary to achieve bounded Lipschitz constants. The similar problem has been further explored by \citet{Wu2023law} and \citet{Das2025direct}, who extended the analysis to more general input distributions and loss functions, respectively. These works focus on lower bounding the Lipschitz constant as a proxy for robustness of specific estimators, whereas our paper establishes both minimax lower and upper bounds on the adversarial risk for all interpolating estimators. For classification, adversarial robustness of interpolating classifiers has been studied in \cite{Belkin2018Overfitting} and \cite{Sanyal2021}.

\emph{Adversarial training of over-parameterized linear model.} In linear regression setting, \cite{Javanmard2020} and \cite{Hassani2024} derived asymptotic risk formulas for adversarially trained estimators by minimizing the empirical adversarial loss over the class of linear functions and two-layer neural networks, respectively. Their results show that in the standard setting, the global minimum of the asymptotic risk occurs in the over-parameterized regime. However, as future $X$-attacks strengthen, over-parameterized models fail to attain the minimum of adversarial risk, a phenomenon they refer to as the \emph{curse of over-parameterization on robustness}. The results in \cite{Javanmard2020, Hassani2024} focus on specific regression estimators with Gaussian assumptions on regressor and random error, and the adversarially trained estimator considered there does not exhibit the interpolation property that is central to understanding the robustness of interpolating DNNs in our work.

\subsection{Outline}

The rest of the paper is organized as follows. We formally introduce the setup of nonparametric regression and 
define interpolating estimators in Section~\ref{sec:setup}. The main results of this paper, which characterize the minimax adversarial risk of interpolating estimators, are presented in Section~\ref{sec:main}. In Section~\ref{sec:simu}, we conduct numerical experiments to verify our theoretical findings. Finally, we discuss 
notable future research directions in Section~\ref{sec:disc}. The proofs of the theoretical results are deferred to the Appendix.

\section{Problem setup}\label{sec:setup}

\subsection{Basic settings}\label{subsec:basic}

Let $(X, Y)$ be a random pair in $\mathcal{Z} \triangleq \mathcal{X} \times \mathcal{Y}$, where $\mathcal{X} \triangleq [0,1]^d$ is the input space, and $\mathcal{Y} \subseteq \mathbb{R}$ is the output space. Denote by $\mathbb{P}_{XY}$ the joint distribution of $(X, Y)$. In this paper, we model the data distribution $\mathbb{P}_{XY}$ in the following nonparametric regression framework
\begin{equation}\label{eq:model}
  Y = f^*(X) + \xi,
\end{equation}
where $X$ follows an unknown marginal distribution $\mathbb{P}_X$ on $\mathcal{X}$, $f^*(x) \triangleq \mathbb{E}(Y|X=x)$ is the regression function, and the random error $\xi$ is $N(0, \sigma^2)$ or sub-Gaussian conditioned on $X$.

Throughout this paper, we impose the following assumption on $\mathbb{P}_X$ for simplicity.
\begin{assumption}\label{ass:bounded_density}
  The marginal distribution $\mathbb{P}_X$ is absolutely continuous with respect to the Lebesgue measure $\lambda$ on $\mathcal{X}$, with a density function satisfying
  \begin{equation*}
    \barbelow{\mu} \leq \frac{d\mathbb{P}_X}{d \lambda}(x) \leq \bar{\mu}
  \end{equation*}
  for all $x \in \mathcal{X}$, where $\barbelow{\mu}$ and $\bar{\mu}$ are two positive constants.
\end{assumption}
This assumption ensures that the input observations are evenly distributed over the compact support, a common assumption in many nonparametric estimation problems \citep[see, e.g.,][]{Stone1982, Audibert2007, Chhor2024Benign}. 

Suppose we are given a sample $Z_n \triangleq \{ (X_i, Y_i),i=1\ldots,n \}$, where $(X_i, Y_i)$ are i.i.d. draws from the distribution $\mathbb{P}_{XY}$ defined by the regression model (\ref{eq:model}), and the random errors are $\xi_i = Y_i - f^*(X_i), i=1,\ldots,n$. Given any $x \in \mathcal{X}$, the objective is to construct an estimator $\hat{f}(x) = \hat{f}(x;Z_n)$ for $f^*(x)$, where $\hat{f}(x;Z_n)$ is a measurable function. In the standard nonparametric framework, the performance of $\hat{f}$ is usually assessed by the $L_2$-risk
\begin{equation*}
  R(\hat{f}, f^*) \triangleq \mathbb{E}\big[\hat{f}(X) - f^*(X)\big]^2,
\end{equation*}
where the expectation $\mathbb{E}$ is taken with respect to both the training sample $Z_n$ and a new input $X$.

Consider the following H\"{o}lder class. Let $\beta >0$, $L >0$, and $\lfloor \beta \rfloor$ denote the largest positive integer strictly smaller than $\beta$.
\begin{definition}
  The class of $(\beta,L)$-H\"{o}lder smooth functions is defined as
\begin{equation}\label{eq:holder}
\begin{split}
   \mathcal{H}(\beta, L) \triangleq & \left\{ f: \mathcal{X} \to \mathbb{R}: \max _{0 \leq|s| 
   \leq\lfloor\beta\rfloor} \sup _{x \in\mathcal{X}}\left|D^s f(x)\right|+\max _{|s|=\lfloor\beta\rfloor} \sup 
   _{x, z \in\mathcal{X}} \frac{\left|D^s f(x)-D^s f(z)\right|}{\|x-z\|^{\beta-\lfloor\beta\rfloor}} \leq L 
   \right\},
\end{split}
\end{equation}
where we used multi-index notation $D^s\triangleq \frac{\partial^{s_1+\cdots + s_d}}{\partial^{s_1} \cdots \partial^{s_d}}$ with $s=(s_1, \ldots, s_d) \in \mathbb{N}^d$ and $|s|\triangleq \sum_{i=1}^{d}|s_i|$. The norm $\|\cdot \|$ denotes the Euclidean norm.
\end{definition}
Our analysis will focus on regression function estimation under the condition that the true regression function $f^*$ belongs to the H\"{o}lder class $\mathcal{H}(\beta, L)$ for some $0<\beta, L < \infty$. Beyond the above definitions, we also introduce some notations that will be used in this paper.

For $1 \leq p \leq \infty$, we define $B_p(x, r) \triangleq \{z : \|z - x\|_p \leq r\}$ as the $\ell_p$-ball centered at $x$ with radius $r$, where $\|\cdot \|_p$ denotes the $\ell_p$ norm. For any positive sequences $a_n$ and $b_n$, we write $a_n = O(b_n)$, $b_n = \Omega(a_n)$, or $a_n \lesssim b_n$ if there exist constants $C > 0$ and $N > 0$ such that for all $n \geq N$, $a_n \leq C b_n$. If both $a_n = O(b_n)$ and $b_n = O(a_n)$ hold, we denote this as $a_n \asymp b_n$ or $a_n = \Theta(b_n)$. We write $a_n = o(b_n)$ and $a_n \sim b_n$ for $\lim_{n \to \infty}a_n/b_n=0$ and $\lim_{n \to \infty}a_n/b_n=1$, respectively. Define $x_+ \triangleq \max(x, 0)$. In addition, let $\phi$ denote the density function of the standard univariate Gaussian distribution, $\Phi$ the corresponding cumulative distribution function, and $\tilde{\Phi} \triangleq 1 - \Phi$ its right-tail complement.

\subsection{Adversarial risks}\label{subsec:adv_risk}

In this paper, we study the regression estimation problem under future $X$-attacks. We measure the robustness of an estimator using the following adversarial $L_2$-risk:
\begin{equation}\label{eq:adv_risk}
  R_r(\hat{f}, f^*) \triangleq \mathbb{E}\sup_{X' \in \mathcal{X} \cap B_p(X,r)}\big[\hat{f}(X') - f^*(X)\big]^2,
\end{equation}
where $1 \leq p \leq \infty$, and $r$ is the radius of the perturbation ball $B_p(X,r)$ satisfying $r = O(1)$. The measure \eqref{eq:adv_risk} has been widely adopted in the adversarial learning literature. Specifically, it has been used to evaluate various adversarially trained estimators \citep{Xing2021, Scetbon2023, Attias2023Real-Valued, Peng2024, Peng2025aos}, to investigate the effect of interpolation on adversarial robustness \citep{Donhauser2021, Hao2024surprising}, and it is also closely related to the objective of adversarial regularization \citep{Zhang2019}.

Note that when $r = 0$, the adversarial risk (\ref{eq:adv_risk}) reduces to the standard $L_2$-risk, i.e., $R_0(\hat{f}, f^*) = R(\hat{f}, f^*)$, which corresponds to the setting without an adversary. The minimax rates of convergence over Hölder and similar classes have been established in various settings \citep[see, e.g.,][]{Stone1982, Yang1999Information, Gyorfi2002, Tsybakov2009}. In the nonparametric regression setting we consider, it is well-known that 
\begin{equation}\label{eq:minimax_1}
  \inf_{\hat{f} }\sup_{f^* \in \mathcal{H}(\beta, L)} R(\hat{f}, f^*) \asymp n^{-\frac{2\beta}{2\beta+d}},
\end{equation}
where the infimum is taken over all measurable estimators. In the general setting where $r \geq 0$, \cite{Peng2025aos} established the minimax rate under the adversarial $L_2$-risk:
\begin{equation}\label{eq:minimax_2}
  \inf_{\hat{f} }\sup_{f^* \in \mathcal{H}(\beta, L)} R_r(\hat{f}, f^*) \asymp r^{2(1 \wedge \beta)} + n^{-\frac{2\beta}{2\beta+d}},
\end{equation}
which precisely characterizes the impact of the adversarial $\ell_p$-ball attack. Note that the order of the adversarial minimax risk (\ref{eq:minimax_2}) can be much worse than (\ref{eq:minimax_1}) as $r$ becomes larger.

\subsection{Interpolating estimator}\label{subsec:interpolation}

The minimax results in (\ref{eq:minimax_1})–(\ref{eq:minimax_2}) are established for all estimators. Motivated by the observation that over-parameterized DNNs typically fit the training data exactly, one naturally wonders if interpolation affects the rate of convergence when subject to the future $X$-attack. Given the data $(X_i, Y_i)_{i=1}^n$, we define
\begin{equation}\label{eq:qinghe1}
  \mathcal{I}(\delta) \triangleq \left\{ f:\max_{1 \leq i \leq n}|f(X_i) - Y_i| \leq \delta\text{\;and $f$ 
  is measurable w.r.t. $x$ and $Z_n$}\right\}
\end{equation}
to be the class of $\delta$-interpolation functions, where $\delta = \delta_n \geq 0$ measures the degree of interpolation.

When $\delta = 0$, the set $\mathcal{I}(0)$ includes the completely interpolating functions discussed in \cite{Belkin2018Overfitting, Belkin2019, Xing2022Benefit, Chhor2024Benign}. Over this class, it has been shown in \cite{Chhor2024Benign} that for all $0<\beta, L < \infty$,
\begin{equation}\label{eq:minimax_3}
  \inf_{\hat{f} \in \mathcal{I}(0)}\sup_{f^* \in \mathcal{H}(\beta, L)} R(\hat{f}, f^*) \asymp 
  n^{-\frac{2\beta}{2\beta+d}},
\end{equation}
which demonstrates that the optimal minimax rate of interpolating estimators under the standard squared $L_2$-risk remains the same as that without the interpolation constraint. This shows that interpolation while suggesting overfitting actually does not damage the classical minimax optimality. In fact, the optimal rate in (\ref{eq:minimax_3}) can be achieved through local averaging regression with a singular kernel \citep{Belkin2018Overfitting, Belkin2019, Chhor2024Benign}. In a related RKHS regression setting, \cite{Haas2023} proved that the optimal rate in the Sobolev class can be attained by a kernel ridgeless estimator with a spiky-smooth kernel function. 

However, for the adversarial risk $R_r(\hat{f}, f^*)$ of the estimator $\hat{f}$ in $\mathcal{I}(\delta)$ with $\delta> 0$ and $r>0$, the following essential questions remain to be understood:
\begin{description}
  \item[Q1.] What is the order of the minimax adversarial risk of $\delta$-interpolating estimators, i.e., 
      $\inf_{\hat{f} \in \mathcal{I}(\delta)}\sup_{f^* \in \mathcal{H}(\beta, L)} R_r(\hat{f}, f^*)$?
  \item[Q2.] How do the attack magnitude $r$, smoothness level $\beta$, and interpolation degree $\delta$ 
      interplay to affect the adversarial minimax rate?
\end{description}
Answers to these questions may have important implications for the adversarial robustness of interpolating 
estimators and could offer insights into the robustness of DNNs.

\section{Main results}\label{sec:main}

\subsection{Minimax rates}

In this subsection, we first derive a lower bound on the minimax adversarial rate for $\delta$-interpolating estimation. Since the minimax lower bound focuses on the negative side of the problem (that is, the limit attainable by any $\delta$-interpolator), we work under the Gaussian assumption of the random errors. 

\begin{theorem}[Lower bound]\label{theo:lower}
   Suppose Assumption~\ref{ass:bounded_density} holds and $\xi_i$ are i.i.d. $N(0,\sigma^2)$. We have the following minimax lower bound:
  \begin{equation}\label{eq:lower_bound}
    \inf_{\hat{f}\in \mathcal{I}(\delta)}\sup_{f^* \in \mathcal{H}(\beta, L)}R_{r}(\hat{f}, f^*) \gtrsim r^{2(1 
    \wedge \beta)} + n^{-\frac{2\beta}{2\beta+d}} +\int_{\mathcal{X}}\mathbb{E}\max_{i \in \mathcal{S}_p(x,r)} 
    (|\xi_{i}| - \delta)_+^2dx,
  \end{equation}
  where $\mathcal{S}_p(x,r) \triangleq \{ i: X_i \in \mathcal{X} \cap B_p(x,r) \}$ denotes the index set of points within the adversarial ball $B_p(x,r)$.

\end{theorem}

The result in Theorem~\ref{theo:lower} has several implications, both practical and theoretical. First, the minimax lower bound (\ref{eq:lower_bound}) serves as a measure of the fundamental limits of interpolating estimators, illustrating the potential damage of interpolation to adversarial robustness. Comparing it with the optimal adversarial rate (\ref{eq:minimax_2}), we observe that interpolation introduces an additional term, $\int_{\mathcal{X}}\mathbb{E}\max_{i \in \mathcal{S}_p(x,r)} (|\xi_{i}| - \delta)_+^2dx$, which depends on the interpolation degree $\delta$ and the attack magnitude $r$. Useful bounds for this term will be established in Section~\ref{subsec:minimax_rate}. Second, (\ref{eq:lower_bound}) provides a benchmark for the best possible robust performance of any interpolating estimator under future $X$-attacks. In particular, if an interpolating estimator attains this benchmark, it can be regarded as the optimal robust interpolating rule.

We next show that the minimax lower bound in Theorem~\ref{theo:lower} is attainable under more general sub-Gaussian assumption on the random errors. Our approach constructs a set of $\delta$-interpolating rules, built upon a regression estimator with a controllable \emph{localized sup-norm risk}. The procedure consists of the following two steps: 
\begin{description}

  \item[Step 1:] Construct a measurable regression estimator $\hat{f}$ such that
    \begin{equation}\label{eq:local_risk}
    \mathbb{E}\sup_{x' \in \mathcal{X} \cap B_p(x,r)} \big[ \hat{f}(x') - f^*(x') \big]^2 \lesssim r^{2\beta} + 
    n^{-\frac{2\beta}{2\beta+d}}
  \end{equation}
  holds for any $x \in \mathcal{X}$, $ 0 \leq r \lesssim 1$, and $f^* \in \mathcal{H}(\beta, L)$. 
  
  \item[Step 2:] Let $0 \leq \tau \lesssim 1$. For any $x \in \mathcal{X}$, recall the index set $\mathcal{S}_p(x,\tau) \triangleq \{ i: X_i \in \mathcal{X} \cap B_p(x,\tau) \}$. Based on $\hat{f}$, consider the interpolating estimators in $\mathcal{I}(\delta)$ satisfying
  \begin{equation}\label{eq:lpe_interpolated}
  \hat{f}_{\delta,\tau}(x) = \left\{\begin{array}{ll}
\hat{f}(x) &\quad x \in \mathcal{X} \setminus \cup_{i = 1}^n B_p(X_i, \tau), \\
\hat{g}(x) & \quad x \in  \mathcal{X} \cap \left\{ \cup_{i = 1}^n B_p(X_i, \tau) \setminus 
\{X_i\}_{i=1}^n \right\}, \\
\hat{f}(X_i)+\psi_{\delta}\big(Y_i - \hat{f}(X_i)\big) &\quad x = 
X_i,\; i =1,\ldots,n, 
\end{array}\right.
\end{equation}
where $\psi_{\delta}(t) \triangleq \mathrm{sgn}(t)(|t|-\delta)_+$, and $\hat{g}(x)$ is any measurable function (not necessarily continuous or differentiable) satisfying  
\begin{equation}\label{eq:lianjie}
  \big| \hat{g}(x) - \hat{f}(x) \big| \leq \max_{i \in \mathcal{S}_p(x,\tau)} \big[ | Y_i - \hat{f}(X_i) | - 
  \delta \big]_+. 
\end{equation}
\end{description}

The condition on the localized sup-norm risk in (\ref{eq:local_risk}) is quite mild and can be achieved by the piecewise local polynomial estimator constructed in \cite{Peng2025aos} under sub-Gaussian assumption. It is also conjectured that other local nonparametric estimators \citep[see, e.g.,][]{Stone1982, BERTIN2004225, GAIFFAS2007782, Tsybakov2009} and series least squares estimators \citep[see, e.g.,][]{CHEN2015447, BELLONI2015345} may also achieve \eqref{eq:local_risk} provided their regularization parameters are chosen properly.

The interpolators defined in (\ref{eq:lpe_interpolated}) involve two key parameters $\delta$ and $\tau$, which govern the interpolation degree and the estimator's behavior within the local neighborhoods of $X_i$, respectively. Specifically, at the training data points $X_i$, $\hat{f}_{\delta,\tau}(X_i)$ can be viewed as a soft-thresholding-type rule applied to the observed data points, which shrinks $Y_i$ towards the base estimator value $\hat{f}(X_i)$ by $\delta$. Within the local neighborhood $B_p(X_i, \tau)$, the condition (\ref{eq:lianjie}) ensures that the difference between the interpolator and the base estimator is bounded by $\max_{i \in \mathcal{S}_p(x,\tau)} \big[ | Y_i - \hat{f}(X_i) | - \delta \big]_+$. Outside all such neighborhoods of $X_i$, $\hat{f}_{\delta,\tau}(x)$ coincides with the base estimator $\hat{f}(x)$. It is easily verified that $\hat{f}_{\delta,\tau}$ is measurable and satisfies the interpolation restriction in (\ref{eq:qinghe1}) with equality. Theorem~\ref{theo:upper} below provides a general upper bound on the adversarial risk of $\hat{f}_{\delta,\tau}$. 

\begin{theorem}[Upper bound]\label{theo:upper}
Suppose Assumption~\ref{ass:bounded_density} holds, and $\xi_i$ are independent sub-Gaussian random variables. 
Then, the adversarial risk of $\hat{f}_{\delta,\tau}$ satisfies
\begin{equation}\label{eq:yuandan3}
  \sup_{f^* \in \mathcal{H}(\beta, L)}R_{r}(\hat{f}_{\delta,\tau}, f^*) \lesssim (r+\tau)^{2(1\wedge \beta)} + 
  n^{-\frac{2\beta}{2\beta+d}} +\int_{\mathcal{X}}\mathbb{E}\max_{i \in \mathcal{S}_p(x,r + \tau)} (|\xi_{i}| - 
  \delta)_+^2dx.
\end{equation}
\end{theorem}

Comparing the upper bound in (\ref{theo:upper}) with the minimax lower bound in Theorem~\ref{theo:lower}, we observe that for $\hat{f}_{\delta,\tau}$ to achieve minimax optimality, it suffices for the neighborhood size $\tau$ to shrink to zero at a certain fast rate. To illustrate this requirement, we proceed to examine two examples. 

\begin{example}[Simple $\delta$-interpolator]\label{exp:simple}
  A simple $\delta$-interpolator corresponds to setting $\tau = 0$. It then follows immediately from 
  (\ref{theo:upper}) that the upper bound for the adversarial risk of $\hat{f}_{\delta,0}$ matches the minimax 
  lower bound given in (\ref{eq:lower_bound}). 
\end{example} 

The construction of the simple $\delta$-interpolator $\hat{f}_{\delta,0}$ can be seen as modifying an optimal estimator $\hat{f}$ forcefully at the training data points to fit the $\delta$-interpolation constraint. The special case $\hat{f}_{0,0}$ was also used in \cite{Belkin2019} to illustrate benign overfitting without future $X$-attacks. Note that under the standard $L_2$-loss, it is not surprising that modifying the values of an estimator on a set of measure zero does not change its standard risk. However, this is no longer the case when the adversarial loss of $\hat{f}_{0,0}$ is considered, as such modifications induce an unavoidable cost $\int_{\mathcal{X}}\mathbb{E}\max_{i \in \mathcal{S}_p(x,r )} |\xi_{i}|^2dx$ relative to the optimal adversarial risk in (\ref{eq:minimax_2}). 

\begin{remark}
  It is worth mentioning that $\hat{f}_{\delta,0}$ can serve as a theoretical tool for constructing a matching upper bound to (\ref{eq:lower_bound}), thereby determining the minimax rate that is key to understand the statistical limits of interpolation. However, this does not necessarily imply that this simple interpolator is intended for practical use in machine learning, given its potential lack of continuity and differentiability.
\end{remark}

\begin{example}[Shrinking-neighborhood $\delta$-interpolator]\label{exp:shrink}
  Suppose $\tau = \tau_n$ with 
  $$
  0 < \tau_n \lesssim n^{-\frac{\beta}{(2\beta+d)(1 \wedge \beta)} \vee \frac{4\beta+d}{d(2\beta+d)}}.
  $$ 
  As proved in Section~\ref{sec:proof_exmaple} of the Appendix, $(r+\tau_n)^{2(1\wedge \beta)} + 
  n^{-\frac{2\beta}{2\beta+d}} \lesssim r^{2(1\wedge \beta)} + n^{-\frac{2\beta}{2\beta+d}}$ and 
  $\int_{\mathcal{X}}\mathbb{E}\max_{i \in \mathcal{S}_p(x,r + \tau_n)} (|\xi_{i}| - \delta)_+^2dx \lesssim 
  n^{-\frac{2\beta}{2\beta+d}}+ \int_{\mathcal{X}}\mathbb{E}\max_{i \in \mathcal{S}_p(x,r)} (|\xi_{i}| - 
  \delta)_+^2dx$. Consequently, the upper bound in (\ref{eq:yuandan3}) matches the minimax lower bound in 
  (\ref{eq:lower_bound}), and $\hat{f}_{\delta,\tau_n}$ is an optimal $\delta$-interpolator under adversarial 
  $X$-attacks.  
\end{example} 

Given $\tau >0$, one can readily construct a continuous (or even smooth) function $\hat{g}$ (e.g., via linear or spline interpolation) such that (\ref{eq:lianjie}) holds. Consequently, Example~\ref{exp:shrink} demonstrates that there exists a broad class of $\delta$-interpolators that can fit the data continuously or even differentiably while attaining the minimax lower bound in Theorem~\ref{theo:lower}. 

\begin{remark}
When $\delta = 0$ and $r = 0$, the results in Example~\ref{exp:shrink} also imply that the exact interpolating 
estimator $\hat{f}_{0,\tau_n}$ attains the optimal rate (\ref{eq:minimax_1}) in the standard setting. Requiring 
$\tau_n$ to shrink to zero fast enough is also adopted in other benign overfitting estimators, such as the 
inflated histogram estimator studied in \citet{Mucke2025}, which corresponds to the construction 
$\hat{f}_{0,\tau}$ in (\ref{eq:lpe_interpolated}) with $\hat{f}$ being a histogram estimator and $\hat{g} = \hat{f}$. The optimal interpolating estimators proposed in \citet{Belkin2018Overfitting, Belkin2019, Haas2023, Chhor2024Benign} indeed share the similar spirit as Example~\ref{exp:shrink}. The singular kernel functions in these works force the interpolator to the optimal estimator $\hat{f}$ once $x$ leaves sufficiently small neighborhoods of the data points. 
\end{remark}

Theorems~\ref{theo:lower}--\ref{theo:upper} together with the specific interpolators in Examples~\ref{exp:simple}--\ref{exp:shrink} establish the minimax rate of convergence for the $\delta$-interpolating estimators under future $X$-attacks:
\begin{equation}\label{eq:minimax_rate}
  \inf_{\hat{f}\in \mathcal{I}(\delta)}\sup_{f^* \in \mathcal{H}(\beta, L)}R_{r}(\hat{f}, f^*) \asymp r^{2(1 \wedge \beta)} + n^{-\frac{2\beta}{2\beta+d}} +\int_{\mathcal{X}}\mathbb{E}\max_{i \in \mathcal{S}_p(x,r)} (|\xi_{i}| - \delta)_+^2dx.
\end{equation}
The minimax rate (\ref{eq:minimax_rate}) addresses Question Q1 posed in Section~\ref{subsec:interpolation}, showing that interpolation entails an additional cost under adversarial risk, namely the term $\int_{\mathcal{X}}\mathbb{E}\max_{i \in \mathcal{S}_p(x,r)} (|\xi_{i}| - \delta)_+^2dx$.

\subsection{Phase transitions}\label{subsec:minimax_rate}

To address Q2, it remains to analyze the order of the integral term $\int_{\mathcal{X}}\mathbb{E}\max_{i \in \mathcal{S}_p(x,r)} (|\xi_{i}| - \delta)_+^2dx$ to determine the rate in (\ref{eq:minimax_rate}). From a technical perspective, it suffices to establish matching upper and lower bounds for $\mathbb{E}\max_{i \in \mathcal{S}_p(x,r)} (|\xi_{i}| - \delta)_+^2$ uniformly over all $x \in \mathcal{X}$. The interplay between $r$ and $\delta$ makes the precise characterization of the order of this term quite involved (see Sections~\ref{sec:proof_3}--\ref{sec:proof_5} of the Appendix for detailed proofs) and leads to a nuanced phase transition phenomenon in the optimal rates. In what follows, we present the minimax rate (\ref{eq:minimax_rate}) under three different regimes for $\delta$.

\begin{theorem}[Low interpolation regime]\label{theo:low}

Suppose Assumption~\ref{ass:bounded_density} holds. Assume the attack magnitude satisfies $0 \leq nr^d \leq cn^{\gamma}$ for some constant $c>0$ and $0<\gamma \leq 1$. Then there exists a constant $C$ depending only on the sub-Gaussian parameter of $\xi_i$ such that if $\delta \geq \sqrt{C(\frac{2\beta}{2\beta+d}+\gamma) \log n}$, we have $\int_{\mathcal{X}}\mathbb{E}\max_{i \in \mathcal{S}_p(x,r)} (|\xi_{i}| - \delta)_+^2dx = O(n^{-\frac{2\beta}{2\beta+d}})$ under the sub-Gaussian assumption. And the minimax rate becomes
      \begin{equation}\label{eq:low}
       \inf_{\hat{f}\in \mathcal{I}(\delta)}\sup_{f^* \in \mathcal{H}(\beta, L)}R_{r}(\hat{f}, f^*) \asymp r^{2(1 \wedge \beta)} + n^{-\frac{2\beta}{2\beta+d}}.
      \end{equation}
In particular, if $\delta \geq \sqrt{C(\frac{4\beta+d}{2\beta+d}) \log n}$, then (\ref{eq:low}) holds for any $0 \leq r \lesssim 1$. 
\end{theorem}

\begin{figure}[!t]
  \centering
  \includegraphics[width=1\linewidth]{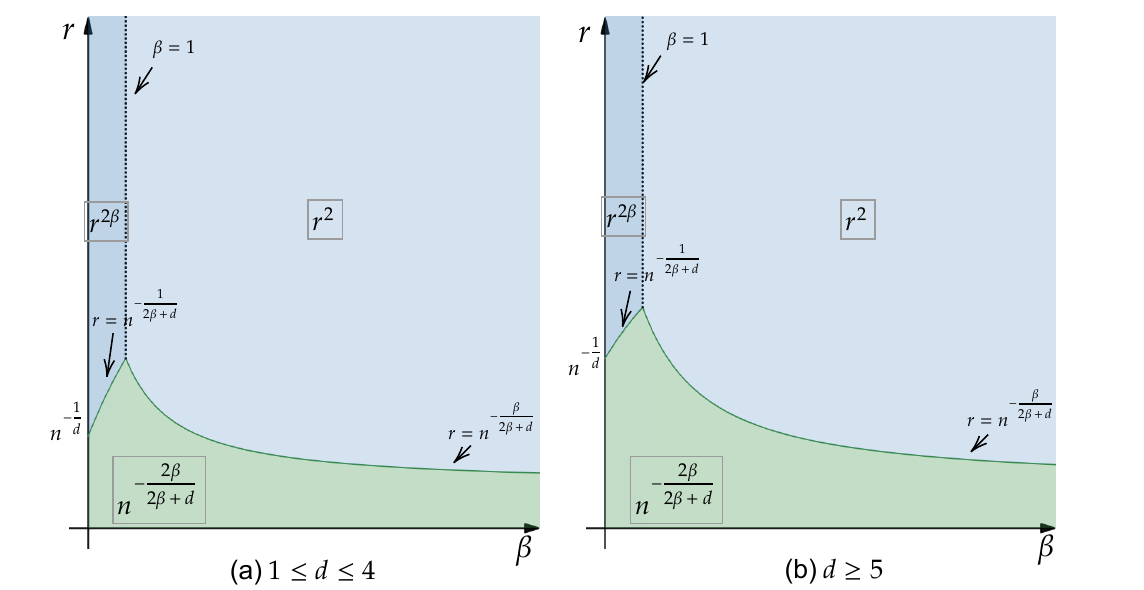}
  \caption{Minimax rate in the low interpolation regime: Panel (a) illustrates the case where $1 \leq d \leq 4$, while panel (b) corresponds to $d \geq 5$. The quantity in the grey rectangle denotes the dominant rates in the respective parameter spaces.}
  \label{fig:low}
\end{figure}

Figure~\ref{fig:low} presents a phase diagram of the minimax rate in Theorem~\ref{theo:low}. It shows that in the low interpolation regime where $\delta$ exceeds a constant multiple of $(\log n)^{1/2}$, the order of (\ref{eq:low}) coincides with that of (\ref{eq:minimax_2}). This indicates that mild interpolation does not compromise adversarial robustness in the minimax sense. In addition, Figure~\ref{fig:low} illustrates that when the attack magnitude is below $n^{-\frac{(1\wedge \beta)}{2\beta+d}}$, the minimax adversarial rate remains identical to the standard minimax rate (\ref{eq:minimax_1}) in the absence of adversarial perturbations. Once $r$ exceeds this critical order, the rate is governed by $r^{2(1 \wedge \beta)}$, which can still converge if $r \to 0$.

Theorem~\ref{theo:low} provides the sufficient condition on $\delta$ under which the interpolating estimator achieves the optimal adversarial rate in (\ref{eq:minimax_2}). Next, we show that requiring $\delta$ to be larger than $(\log n)^{1/2}$ (in order) is almost necessary to achieve optimal adversarial robustness unless $r$ is extremely small. 

\begin{theorem}[Moderate interpolation regime]\label{theo:Moderate}

Suppose Assumption~\ref{ass:bounded_density} holds and $\xi_i$ are i.i.d. $N(0,\sigma^2)$. Consider the regime in which $C_1 \leq \delta \leq C_2(\log n)^{t}$ for constants $C_1, C_2>0$ and $0< t < 1/2$. Then, 
\begin{equation*}
  \inf_{\hat{f}\in \mathcal{I}(\delta)}\sup_{f^* \in \mathcal{H}(\beta, L)}R_{r}(\hat{f}, f^*) \gtrsim r^{2(1 
  \wedge \beta)} + n^{-\frac{2\beta}{2\beta+d}} + \rho_n \cdot \min(nr^d, 1),
\end{equation*}
where $\rho_n \triangleq \exp[ - C_2^2(\log n)^{2t}/(2\sigma^2) ]$ decays more slowly than any polynomial rate $n^{-a}$ with any fixed $a>0$. 

\end{theorem}

Theorem~\ref{theo:Moderate} shows that when $r$ is not extremely small (i.e., $r \gtrsim n^{-1/d}$), the adversarial risk of $\delta$-interpolation with $\delta \lesssim (\log n)^t$ and $t<1/2$ cannot converge faster than any polynomial rate. This implies that basically the condition $\delta \gtrsim (\log n)^{1/2}$ is necessary for $\delta$-interpolation to be robust under future $X$-attacks when $r \gtrsim n^{-1/d}$.

Finally, we show that as the interpolation degree increases further (equivalently, as $\delta$ deceases further), the adversarial robustness of $\delta$-interpolating estimator exhibits an even more significant deterioration. 

\begin{theorem}[High interpolation regime]\label{theo:high}
  Suppose Assumption~\ref{ass:bounded_density} holds and $\xi_i$ are i.i.d. $N(0,\sigma^2)$. Consider the case 
  where $0 \leq \delta \leq C_3\sigma $ for some fixed positive constant $C_3>0$. In this case, when $nr^d \to 0$, we have $\int_{\mathcal{X}}\mathbb{E}\max_{i \in \mathcal{S}_p(x,r)} (|\xi_{i}| - \delta)_+^2dx \asymp nr^d$, and the minimax risk satisfies
\begin{equation*}
  \inf_{\hat{f}\in \mathcal{I}(\delta)}\sup_{f^* \in \mathcal{H}(\beta, L)}R_{r}(\hat{f}, f^*) \asymp r^{2(1 
  \wedge \beta)} + n^{-\frac{2\beta}{2\beta+d}} + nr^d.
\end{equation*}
When $nr^d  \gtrsim 1$, we have $\int_{\mathcal{X}}\mathbb{E}\max_{i \in \mathcal{S}_p(x,r)} (|\xi_{i}| - 
\delta)_+^2dx \gtrsim 1$, and
$$ \inf_{\hat{f}\in \mathcal{I}(\delta)}\sup_{f^* \in \mathcal{H}(\beta, L)}R_{r}(\hat{f}, f^*) = \Omega(1) .$$
When $nr^d > C_4 \log n$, where $C_4$ is a constant depending on $\barbelow{\mu}$, $p$, and $d$, then we have $\int_{\mathcal{X}}\mathbb{E}\max_{i \in \mathcal{S}_p(x,r)} (|\xi_{i}| - \delta)_+^2dx \,\asymp \log(nr^d)$, and the minimax rate becomes
\begin{equation*}
  \inf_{\hat{f}\in \mathcal{I}(\delta)}\sup_{f^* \in \mathcal{H}(\beta, L)}R_{r}(\hat{f}, f^*) = 
  \Theta\big[\log(nr^d)\big].
\end{equation*}
\end{theorem}

\begin{figure}[!t]
  \centering
  \includegraphics[width=1\linewidth]{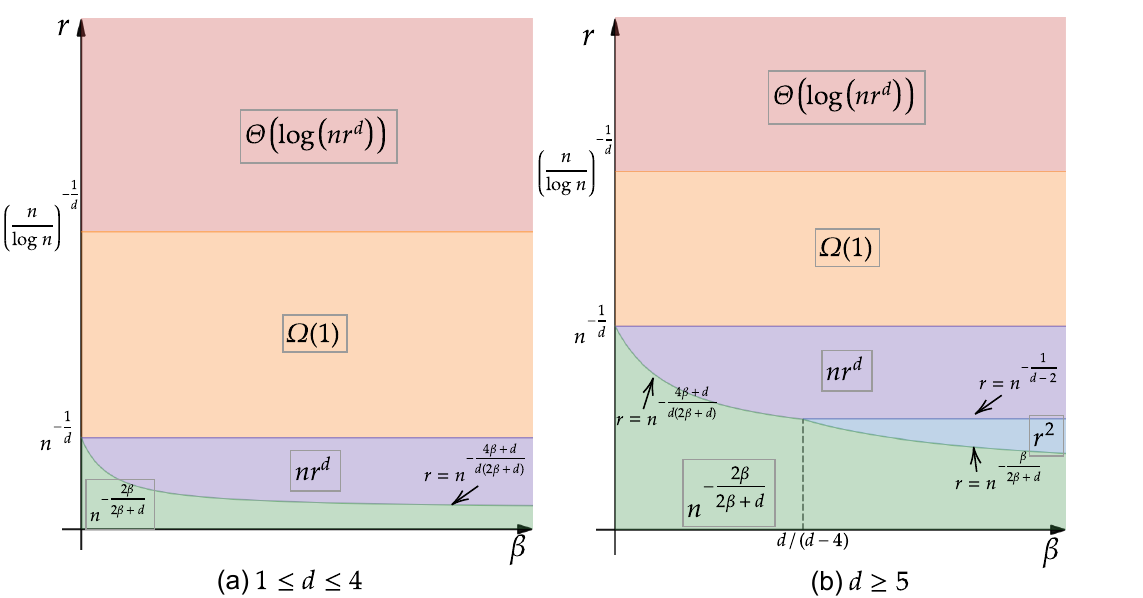}
  \caption{Minimax rate in the high interpolation regime: Panel (a) illustrates the case where $1 \leq d \leq 4$, 
  while panel (b) corresponds to $d \geq 5$. The quantity in the grey rectangle denotes the dominant rates in the 
  respective parameter spaces.}
  \label{fig:high}
\end{figure}

Figure~\ref{fig:high} displays the phase diagram of the minimax rate established in Theorem~\ref{theo:high}, where the interpolation level $\delta$ is upper bounded by a positive constant. Overall, the patterns in Figure~\ref{fig:high} demonstrate that the optimal adversarial rates in the high interpolation regime deteriorate significantly compared to those in the low interpolation regime, highlighting the vulnerability of highly interpolating estimators to adversarial perturbations in future inputs.

Specifically, the critical order of $r$ required to retain the standard minimax rate $n^{-\frac{2\beta}{2\beta+d}}$ becomes substantially lower. For instance, when $1 \leq d \leq 4$, this threshold decreases from $n^{-\frac{(1 \wedge \beta)}{2\beta + d}}$ in the low interpolation case to $n^{-\frac{4\beta + d}{d(2\beta + d)}}$ in the high interpolation regime. Second, as illustrated in Figure~\ref{fig:low}, the $r^2$ term, which dominates the minimax rate over a large part of the parameter space in the low interpolation case, only appears in Figure~\ref{fig:high} when $d \geq 5$ and for smooth function classes with $\beta > \frac{d}{d-4}$. In contrast, the slower rate $nr^d$ can dominate the overall minimax rate in Figure~\ref{fig:high}. Third, when $r \gtrsim n^{-1/d}$, the minimax adversarial rate ceases to converge, and more severely, when $r$ exceeds $(\frac{n}{\log n})^{-1/d}$, the adversarial risk of highly interpolating estimators diverges to infinity. Since $f^* \in \mathcal{H}(\beta,L)$ is a bounded function, a $\delta$-interpolating estimator with bounded $\delta$ cannot even outperform the trivial estimator $\hat{f}_0 \equiv 0$ in this case, as $\hat{f}_0$'s adversarial risk under an future $X$-attack remains bounded.

\subsection{An interesting phenomenon in adversarial learning under future $X$-attacks: curse of sample size and 
blessing of dimensionality}\label{subsec:remark}

In this subsection, we reveal and discuss several interesting yet counterintuitive effects of $n$ and $d$ on the 
minimax adversarial rates in the interpolation regimes.

Consider increasing the sample size from $n$ to $n^{\eta}$ for some $\eta>1$, while retaining the 
attack strength $r$. In the low interpolation regime, the optimal rate of convergence shifts from 
$r^{2(1\wedge\beta)}+n^{-\frac{2\beta}{2\beta+d}}$ to $r^{2(1\wedge\beta)}+n^{-\frac{2\beta \eta}{2\beta+d}}$. The latter rate is faster compared with that at the sample size $n$ if $r =o(n^{-\frac{\beta 
}{(2\beta+d)(1\wedge\beta)}})$. Otherwise, the rate stays the same, with the attack effect term 
$r^{2(1\wedge\beta)}$ dominating. The fact here that as the sample size increases the rate of 
convergence cannot be hurt is fully expected. However, this is no longer true when a highly interpolating estimator is employed. When the sample size increases from $n$ to $n^{\eta}$ for some $\eta>1$ while fixing $r = (\frac{n}{\log(n)})^{-1/d}$, the adversarial risk increases from $\Theta(\log\log(n))$ to $\Theta(\log(n))$. This curse of sample size alerts the potential damage of interpolation when a learning algorithm is under future $X$-attack. The underlying mechanism is intuitive: as the sample size increases, interpolating the data tends to produce a more spiky function, which becomes more susceptible to perturbation in inputs.

In contrast, higher dimensionality may sometimes mitigate the impact of interpolation on adversarial robustness. 
As shown in Figure~\ref{fig:high}, the critical value of $r$ separating converging from non-converging rates is of 
order $n^{-1/d}$. If the sample size is fixed and the dimension increases from $d$ to $k d$ for an integer $k>1$, 
then this critical value grows. This observation indicates that in high-dimensional settings, it is relatively 
easier to preserve the converging adversarial rates. The underlying intuition is that a higher-dimensional input 
space offers more room to disperse data points. As a result, the variability of interpolating estimators can be 
reduced. See also \cite{Wu2023law} for a related phenomenon regarding the lower bounds on the Lipschitz constants 
of interpolating functions.

\subsection{Implications for adversarial robustness in DNNs}

The minimax theory developed in Theorems~\ref{theo:low}--\ref{theo:high} provides a procedure-independent limit of 
the adversarial robustness of $\delta$-interpolating estimators. An important implication is: over-parameterized 
DNNs, which have been empirically observed and theoretically shown to interpolate the training data, are 
fundamentally incapable of achieving robustness against future $X$-attacks. 

\begin{corollary}
  Suppose $\hat{f}_{\mathrm{DNN}}$ is a DNN estimator that nearly interpolates the training data in the sense  
  \begin{equation}\label{eq:inter_cond}
    n^{-1}\sum_{i=1}^{n}[Y_i - \hat{f}_{\mathrm{DNN}}(X_i)]^2 \leq n^{-1}C_3^2\sigma^2, 
  \end{equation}
   which implies $\hat{f}_{\mathrm{DNN}} \in \mathcal{I}(C_3\sigma)$ where $C_3$ is the constant given in Theorem~\ref{theo:high}. Then, its worst case adversarial risk $\sup_{f^* \in \mathcal{H}(\beta, 
   L)}R_{r}(\hat{f}_{\mathrm{DNN}}, f^*)$ is lower bounded by the minimax rates established in 
   Theorem~\ref{theo:high}.
\end{corollary}

Condition~(\ref{eq:inter_cond}) can be satisfied for DNN estimators trained via gradient descent in a nonparametric regression setting similar to ours \citep[see, e.g., Theorem 1 of][]{Kohler2021}. A real data example will be provided in Section~\ref{sec:simu} to illustrate the impact of interpolation, together with effects of the input dimension.

It is worth mentioning that several kernel interpolating estimators and neural network estimators are proven to be inconsistent under the standard $L_2$-risk \citep[see, e.g.,][]{rakhlin2019consistency, Kohler2021, buchholz2022kernel, Li2023interpolation, Haas2023}. Although these results also imply an inconsistent lower bound for the adversarial $L_2$-risk, they cannot lead to the full spectrum of adversarial rates established in Theorems~\ref{theo:low}--\ref{theo:high} and, more importantly, cannot rule out the possibility of achieving adversarial robustness by improving network architecture. In fact, it has been shown that with suitable modifications to the activation function, interpolating neural networks can recover the optimal standard $L_2$-risk \citep[see Section 5 of][]{Haas2023}. In contrast, our theory provides a fundamental understanding that applies to any interpolating procedure under future $X$-attacks. It demonstrates that achieving optimal adversarial robustness is hopeless for any $\delta$-interpolating neural networks with bounded $\delta$, regardless of any improvement in network architecture.

\section{Numerical experiments}\label{sec:simu}

In this section, we conduct several numerical experiments to illustrate the theoretical results presented in 
Section~\ref{sec:main}. We first consider the univariate case ($d = 1$) in simulations, as it provides a basic 
setting to compare classical nonparametric methods. We then examine a real data example with multivariate 
features, focusing on the robustness of neural network.

\subsection{Synthetic simulation}

The regression function is specified under the following three scenarios:
\begin{description}
  \item[Case 1] $f^*(x) = x^3 - x$;
  \item[Case 2] $f^*(x) = x+ \cos(3x)$;
  \item[Case 3] $f^*(x) = \exp(-x^2) \sin(5x)$.
\end{description}
The input variable $X$ is uniformly distributed on $[-2, 2]$, and the random error term follows a Gaussian 
distribution $N(0, 0.5)$. For each case, we generate training datasets with sample sizes $n_{\text{train}} = 80, 
150, 300$.

We compare several nonparametric methods. Note that the true regression functions in all three cases belong to the 
H\"{o}lder class with arbitrary smoothness $\beta$. We consider the first method (LP) as the local polynomial 
regression estimator with a rectangular kernel and polynomial degree $\ell = 7$. The second method (IP$_1$) is the 
interpolating estimator defined in (\ref{eq:lpe_interpolated}) with $\tau=0$, $\delta = 0.75\sqrt{\log\log n}$, 
and $\hat{f}$ taken as the LP estimator. The choice of $\delta$ in IP$_1$ yields a moderately interpolating 
estimator. The third method (IP$_2$) is the same as IP$_1$ except with $\delta = 0.3$, resulting in a highly 
interpolating estimator based on the LP estimator. The fourth method (SI) is the local polynomial estimator with a 
singular kernel $K(u) = |u|^{-0.2}(1 - |u|)_+^2$, which has been theoretically shown to exactly interpolate the 
training data while simultaneously achieving minimax optimality in standard $L_2$-risk \citep{Chhor2024Benign}. 
For all the local polynomial estimators, bandwidth selection is performed using an independently generated 
validation set of size $n_{\text{vali}} = 100$. We consider a $1$-nearest neighbor regression estimator (1-N) as 
another classical interpolating method for comparison. 

Additionally, we include neural networks in our comparison. Specifically, we consider an overparameterized 
two-layer ReLU network with a skip connection, defined by the predictor $f_{\theta,a_0,b_0}:\mathbb{R} \to 
\mathbb{R}$, 
$$
f_{\theta, a_0, b_0}(x)=\sum_{j=1}^m a_j\left(w_j x+b_j\right)_{+}+a_0 x+b_0,
$$
where $\theta \in \mathbb{R}^{3m}$ denotes the parameters $\{a_j, w_j, b_j\}_{j=1}^m$. To allow for 
overparameterization and interpolation, we do not restrict the width $m$. The estimator is learned by minimizing 
the norm of the weights subject to interpolation constraints:
\begin{equation}\label{eq:MNN}
  \hat{f}_{\mathrm{MNN}}=\argmin _{f_{\theta, a_0, b_0}}\|\theta\|^2 \quad \text { s.t. } \quad \forall\, 1 \leq i 
  \leq n,\; f_{\theta, a_0, b_0}(X_i)=Y_i.
\end{equation}
This minimum-norm neural estimator (MNN) has been adopted in the literature to study the phenomenon of benign 
overfitting \citep[see, e.g.,][]{savarese2019infinite, Ergen2021, hanin2021ridgeless, boursier2023penalising, 
joshi2024noisy}.

To evaluate the performance of the six methods, we generate a testing dataset with $n_{\text{test}} = 100$ and 
compute the adversarial loss defined by
\begin{equation}\label{eq:xishu}
  n_{\text{test}}^{-1} \sum_{i=1}^{n_{\text{test}}} \max_{x' \in [-2 \vee (x_i - r), 2\wedge (x_i+r)]}[f^*(x_i) - 
  \hat{f}(x')]^2,
\end{equation}
where the future $X$-attack magnitude $r$ varies from $0$ to $0.1$. Since the inner maximum in (\ref{eq:xishu}) 
does not admit a closed-form solution in our setting, we approximate it by computing the maximum over a dense grid 
within the interval $[-2 \vee (x_i - r), 2 \wedge (x_i + r)]$. We repeat the above simulation procedure $R = 100$ 
times and report the median of (\ref{eq:xishu}) over the replications for each method. The results are displayed 
in Figure~\ref{fig:risk_simu}.

\begin{figure}[t]
\centering
\vspace{-0.35cm}	
%\subfigtopskip=2pt
%\subfigbottomskip=2pt
%\subfigcapskip=-5pt
\setlength{\abovecaptionskip}{2pt}

\subfigure[(a) Case 1]{
\includegraphics[width=1\linewidth]{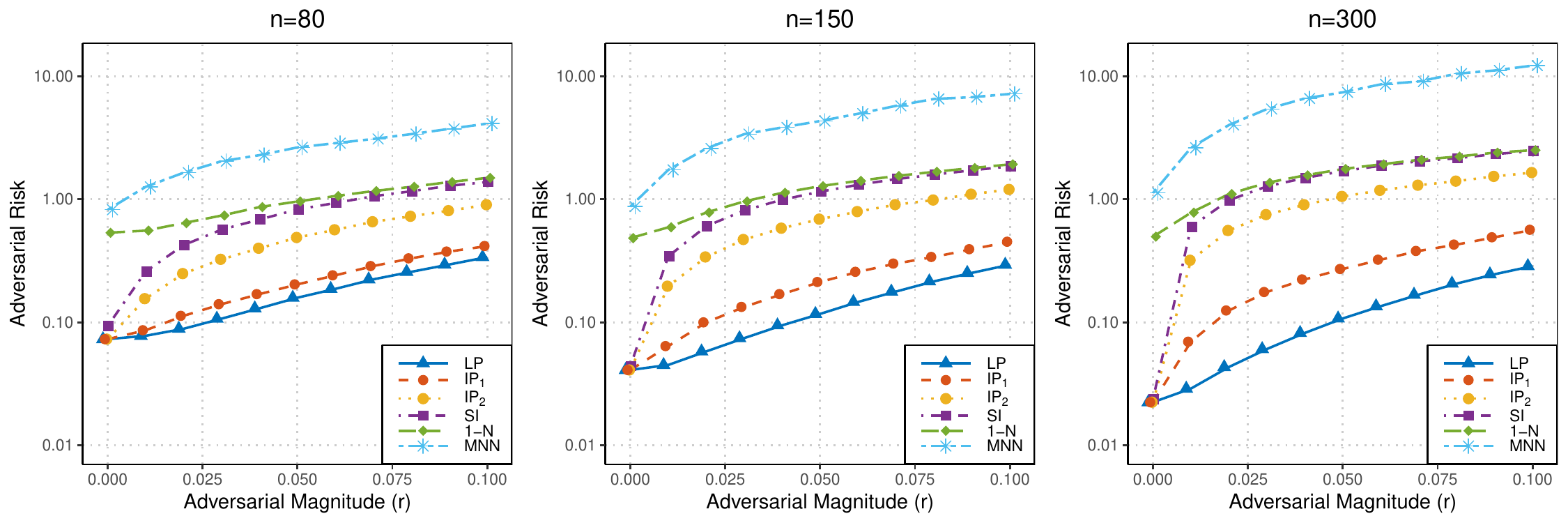}}%加大宽度

\subfigure[(b) Case 2]{
\includegraphics[width=1\linewidth]{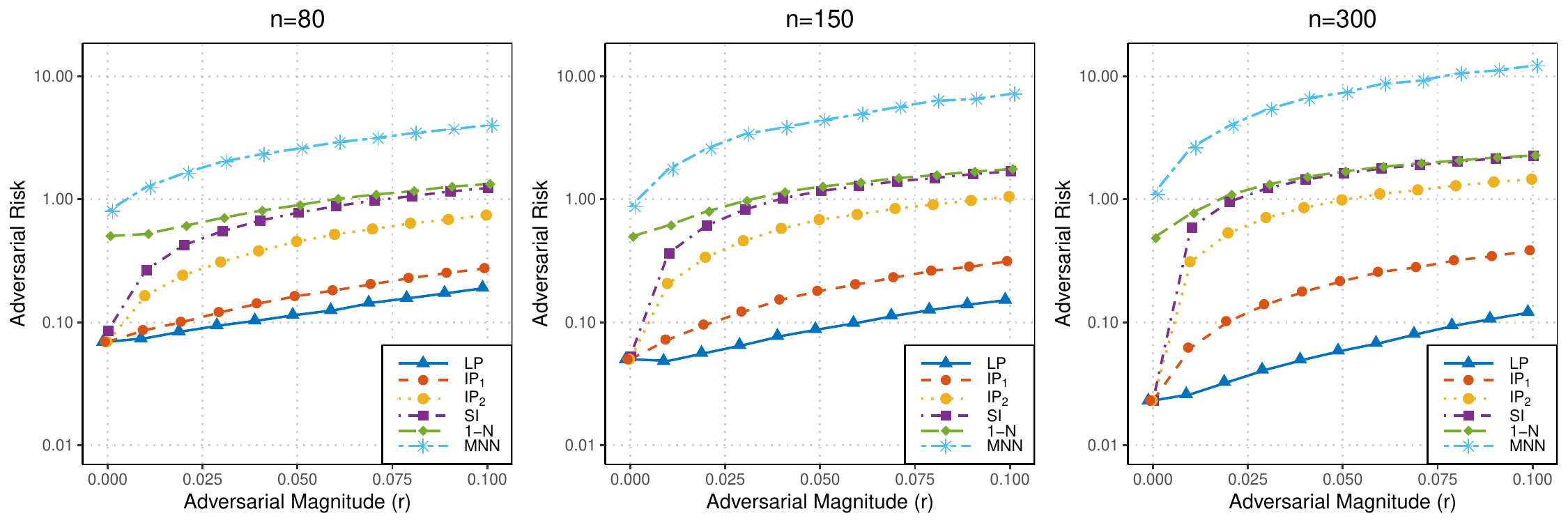}}%加大宽度

\subfigure[(c) Case 3]{
\includegraphics[width=1\linewidth]{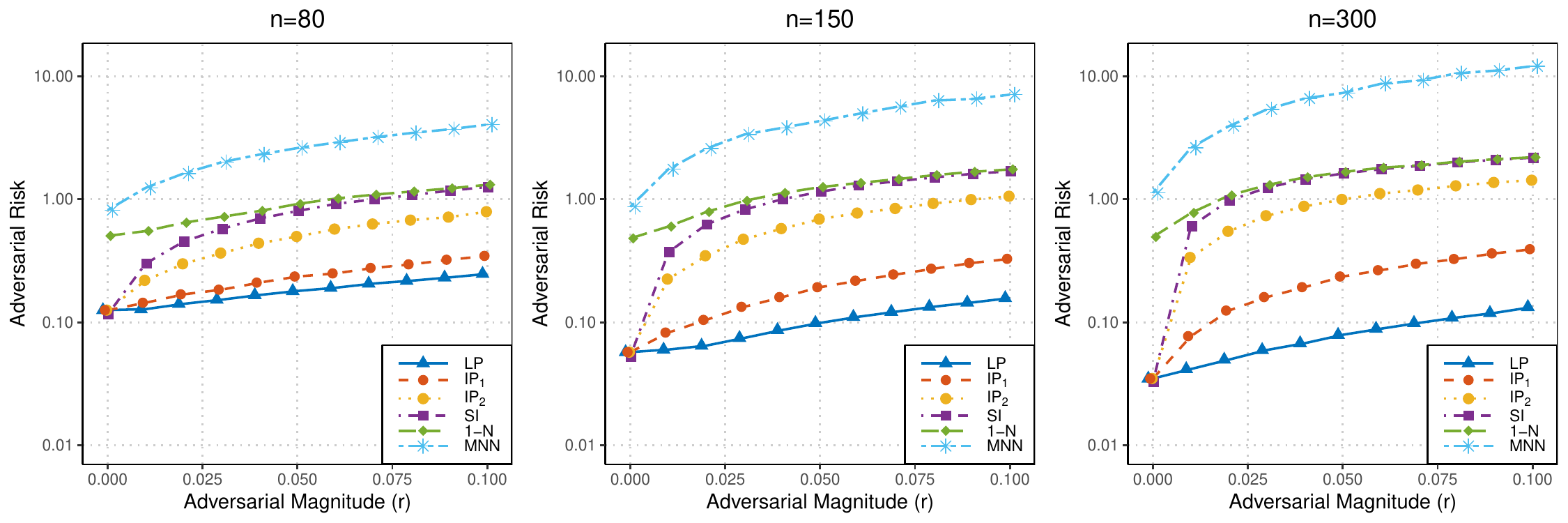}}%加大宽度

\caption{Adversarial risk of the five competing methods: results for Case 1 are shown in row (a), Case 2 in row 
(b), and Case 3 in row (c).}
\label{fig:risk_simu}
\end{figure}

From Figure~\ref{fig:risk_simu}, we observe that the classical local polynomial (LP) estimator and its mildly 
interpolating variant (IP$_1$) exhibit greater stability under future $X$-attacks compared to the other three 
highly interpolating estimators: IP$_2$, SI, ,1-N, and MNN. For example, in Case 1 with $n = 80$ and $r = 0$, the 
median adversarial losses (with standard error in parentheses) for LP, IP$_1$, IP$_2$, SI, ,1-N, and MNN are 0.072 
(0.021), 0.072 (0.021), 0.072 (0.021), 0.093 (0.010), 0.534 (0.009), and 0.823 (0.822), respectively. However, 
when the attack radius increases to $r = 0.05$, their corresponding adversarial losses rise to 0.158 (0.043), 
0.203 (0.043), 0.489 (0.043), 0.827 (0.017), 0.959 (0.014), and 2.639(2.155), respectively. These results 
highlight that while some interpolating estimators (e.g., SI) may perform reasonably well under the standard 
setting ($r = 0$), they become significantly more vulnerable as the strength of the future $X$-attack increases. 
This observation aligns with our theoretical findings: interpolating estimators are suboptimal in terms of 
adversarial robustness when the attack magnitude exceeds certain critical values. Additionally, the 1-N method 
performs poorly even in the absence of adversarial perturbations, which is consistent with existing theory that 
achieving optimality requires increasing the number of neighbors \citep[see, e.g.,][]{Gyorfi2002}. Moreover, the 
MNN estimator also exhibits a higher risk compared to the estimators with optimal rates (such as LP and SI) in the standard setting with $r=0$. This phenomenon has also been empirically observed and studied by 
\cite{Mallinar2022}, who defined the overfitting behavior of neural networks as \emph{tempered overfitting}, which is an overfitting that lies between benign and catastrophic. For further discussion on this phenomenon, see \cite{joshi2024noisy}. 

Another important observation concerns the effect of sample size on adversarial risk. For the non-interpolating 
estimator (LP), we observe that the adversarial risk consistently decreases as the sample size increases. For 
instance, in Case 2 with $r = 0.1$, the adversarial risk of the LP method decreases from 0.191 (0.023) to 0.121 
(0.005) as $n$ increases from 80 to 300, indicating that a larger sample size continues to enhance adversarial 
robustness of non-interpolating estimators. In contrast, the behavior of highly interpolating estimators differs 
significantly. In the same case, the adversarial loss of the SI method increases from 1.233 (0.036) to 2.246 
(0.026) as the sample size grows from 80 to 300. This phenomenon suggests that when the attack magnitude is 
strong, increasing the amount of data can actually harm the adversarial robustness of interpolating estimators. 
Based on the theoretical insight in Section~\ref{subsec:remark}, this counterintuitive phenomenon (i.e., the curse 
of sample size) is expected.

\subsection{Real data illustration}\label{sec:real_data}

To further illustrate the theoretical insights in Section~\ref{sec:main}, we conduct an empirical study on the 
Auto MPG dataset \citep{Quinlan1993} using a neural network and examine its robustness as the training error 
approaches zero. The Auto MPG is a benchmark dataset in nonparametric regression, which contains 392 observations 
with 7 continuous features and one categorical variable. The task is to predict miles per gallon (mpg) based on 
the remaining features. We randomly select $n_{\mathrm{train}} = 300$ observations for training and use the 
remaining $n_{\mathrm{test}} = 92$ for testing. All continuous predictors are standardized to $[0,1]$, and the 
categorical variable is encoded as dummy variables. This train-test split is repeated 100 times to evaluate the 
performance of estimators. 

We consider two settings to examine the effect of dimensionality $d$ on adversarial robustness. In Setting 1 
($d=2$), the two most relevant continuous predictors, weight and horsepower, are used to predict mpg. In Setting 2 
($d=7$), all predictors are included in the model. In both settings, we employ a feedforward neural network with 
three hidden layers of sizes 512, 256, and 128. The network is trained for up to 1000 epochs using gradient 
descent with a learning rate of 0.001. At each epoch, we record the training error, the standard prediction error, 
and the adversarial prediction error, defined as
\begin{equation*}
  n_{\mathrm{test}}^{-1}\sum_{i=1}^{n_{\mathrm{test}}}\max_{x'\in [0,1]^d \cap B_{\infty}(X_i,r)}[Y_i - 
  \hat{f}_{\mathrm{NN}}(x')]^2, 
\end{equation*}
of the neural network estimator $\hat{f}_{\mathrm{NN}}$. Figure~\ref{fig:real_1} reports the medians of these 
errors over the 100 random train-test splits.

\begin{figure}[!t]
  \centering
  \includegraphics[width=1\linewidth]{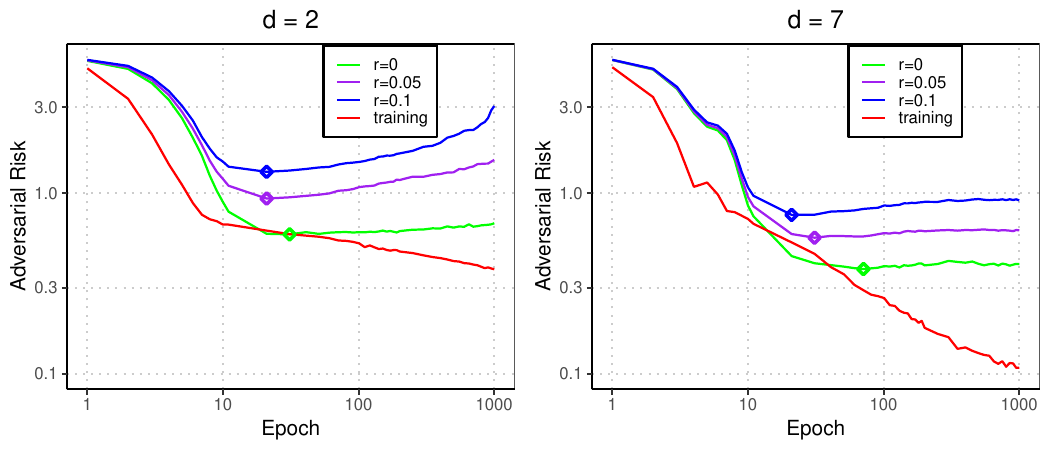}
  \caption{Standard and adversarial risks of the neural network across epochs. The $\diamond$ symbols indicate the 
  minimum value along each risk curve.} 
  \label{fig:real_1}
\end{figure}

From Figure~\ref{fig:real_1}, we observe that as the training error of the neural network gradually decreases, the 
standard prediction error shows at most a mild upward trend, whereas the adversarial risk increases much more 
noticeably. For example, in the case $d = 2$, from epoch 30 to 1000, the standard risk rises only slightly from 
0.595 (0.007) to 0.681 (0.253), corresponding to a $14\%$ increase. In comparison, the adversarial risk with $r = 
0.1$ increases from 1.313 (0.010) to 3.042 (1.213), more than doubling over this range of epochs. This phenomenon 
supports our theoretical understanding that overfitted models may compromise robustness under future $X$-attacks. 
In the case $d = 7$, we observe that both standard and adversarial risks are lower compared to $d = 2$, since more 
informative features are included. After comparing the relative magnitudes of standard and adversarial risks, we 
find that the inflation of adversarial risk is less pronounced when the training error is small compared to that 
in the $d = 2$ case. This phenomenon aligns with Section~\ref{subsec:remark}, which states that high 
dimensionality makes it easier for interpolators to maintain the standard minimax rate.

\begin{figure}[!t]
  \centering
  \includegraphics[width=1\linewidth]{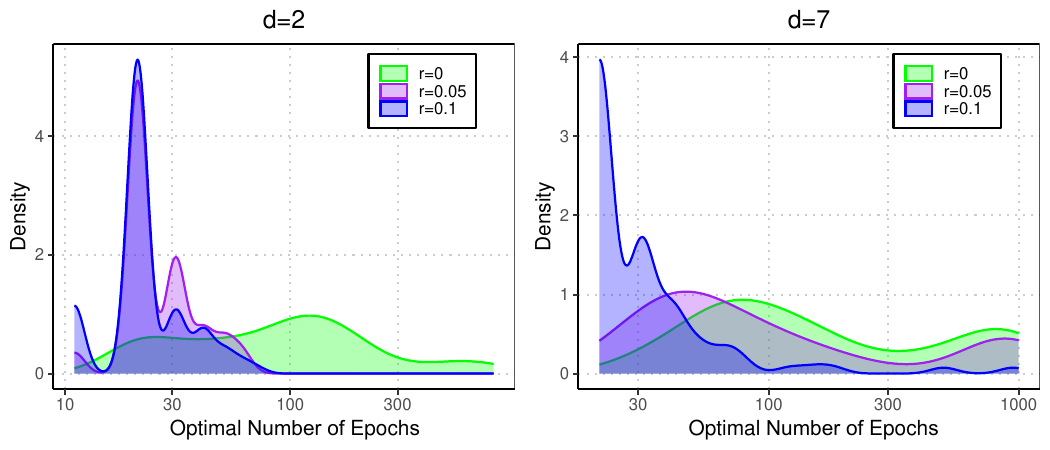}
  \caption{Density of the optimal number of epochs at which the minimum error is attained over 100 replications.} 
  \label{fig:real_2}
\end{figure}

Another important observation from Figure~\ref{fig:real_1} concerns the optimal number of epochs for attaining 
minimum risk. The median of the adversarial risk reaches its minimum earlier than that of the standard risk during 
neural network training. For example, in the case of $d=2$, the epochs at which the standard risk and adversarial 
risks attain their minima are 30, 20, and 20, respectively. We also plot the density of the optimal number of 
epochs for both standard and adversarial risks in Figure~\ref{fig:real_2}. The results show that the neural 
network tends to achieve optimal robustness when training is stopped considerably earlier than what is optimal for 
minimizing the standard risk.

\section{Discussion}\label{sec:disc}

Understanding why DNNs generalize well is a central problem in deep learning theory. As a key aspect of reliable generalization, adversarial robustness has received growing attention in recent years. Motivated by strong empirical and theoretical evidence that training over-parameterized DNNs often yields solutions that fit the training data exactly, we investigate the adversarial robustness of interpolating estimators and examine whether interpolation contributes to the fragility of DNNs under future adversarial $X$-attack or covariates measured with errors. Our minimax risk analysis reveals that highly interpolating estimators fail to attain the minimax optimal adversarial rate. In particular, the phase transition of interpolators departing from the standard nonparametric rate can occur at perturbation magnitudes much smaller than those of regular estimators.

An important direction for future research is to develop theoretically grounded methods for enhancing the robustness of interpolating DNNs. Existing studies have explored various regularization techniques, such as adversarial training \citep{Goodfellow2015, Madry2018Towards}, dropout \citep{Srivastava2014}, and early stopping \citep[see, e.g.,][and references therein]{Ji2021, kohler2025rate}, to mitigate overfitting in neural network training. However, the theoretical foundations for the robust generalization of these practices remain incomplete. Our numerical results also suggest that achieving adversarial robustness may require stopping gradient descent considerably earlier than the point that minimizes the standard risk. A theoretical understanding of the adversarial robustness properties of such regularization methods remains an important direction for future research.

\newpage
\appendix

\section*{Appendix}\label{appendix}
%\addcontentsline{toc}{section}{Appendix}
%\renewcommand{\theequation}{A.\arabic{equation}}
\renewcommand{\thesection}{A.\arabic{section}}
\numberwithin{equation}{section}

\section{Proof of Theorem~\ref{theo:lower}}

When deriving the minimax lower bounds, we focus on the setting in which $\xi_i$ are i.i.d. $N(0,\sigma^2)$. Since the class $\mathcal{I}(\delta)$ is included in the class of all measurable estimators, based on Theorem~4 in \cite{Peng2025aos}, we obtain 
\begin{equation}\label{eq:adversairla_op}
  \inf_{\hat{f}\in \mathcal{I}(\delta)}\sup_{f^* \in \mathcal{H}(\beta, L)}R_{r}(\hat{f}, f^*) \geq 
  \inf_{\hat{f}}\sup_{f^* \in \mathcal{H}(\beta, L)}R_{r}(\hat{f}, f^*) \gtrsim r^{2(1 \wedge \beta)} + 
  n^{-\frac{2\beta}{2\beta+d}}. 
\end{equation}

To establish the lower bound $\int_{\mathcal{X}}\mathbb{E}\max_{i \in \mathcal{S}_p(x,r)}(|\xi_{i}| - \delta)_+^2dx$ in Theorem~\ref{theo:lower}, we restrict the function class to $\{f_0 \}$, where $f_0\equiv0$. It is straightforward to verify that $\{f_0 \} \subset \mathcal{H}(\beta, L)$ for all $0 <\beta < \infty$. Then, the adversarial minimax risk is lower bounded as follows:
\begin{equation}\label{eq:goal2}
\begin{split}
\inf_{\hat{f}\in \mathcal{I}(\delta)}\sup_{f^* \in \mathcal{H}(\beta, L)}R_{r}(\hat{f}, f^*) & \geq 
\inf_{\hat{f}\in \mathcal{I}(\delta)}\sup_{f^* \in \{f_0 \}}R_{r}(\hat{f}, f^*) = \inf_{\hat{f}\in 
\mathcal{I}(\delta)}R_{r}(\hat{f}, f_0)  \\
    & = \inf_{\hat{f}\in \mathcal{I}(\delta)} \mathbb{E}\int_{\mathcal{X}}\sup_{x' \in \mathcal{X} \cap B_p(x,r)} \big[\hat{f}(x') \big]^2  \mathbb{P}_X(dx)\\
     & \geq \barbelow{\mu} \inf_{\hat{f}\in \mathcal{I}(\delta)} \mathbb{E}\int_{\mathcal{X}}\sup_{x' \in \mathcal{X} \cap B_p(x,r)} \big[\hat{f}(x')  \big]^2   dx,
\end{split}
\end{equation}
where the last inequality follows from $\mathcal{X} = [0,1]^d$ and Assumption~\ref{ass:bounded_density}.

Since $f_0\equiv0$, the interpolation constraint $\hat{f} \in \mathcal{I}(\delta)$ implies that 
\begin{equation}\label{eq:gogo_1}
  |\hat{f}(X_{i}) - \xi_{i}| \leq \delta, \quad i =1,\ldots, n,
\end{equation}
where $\xi_i$ i.i.d. $N(0, \sigma^2)$. Recall the index set $\mathcal{S}_p(x,r) = \{ i: X_i \in \mathcal{X} \cap 
B_p(x,r) \}$. Then, the last term in (\ref{eq:goal2}) is lower bounded by
\begin{equation}\label{eq:duzi1}
\begin{split}
   \inf_{\hat{f}\in \mathcal{I}(\delta)} \mathbb{E}\int_{\mathcal{X}}\sup_{x' \in \mathcal{X} \cap B_p(x,r)} 
   \big[\hat{f}(x') \big]^2   dx & \geq \inf_{\hat{f}\in \mathcal{I}(\delta)} \mathbb{E}\int_{\mathcal{X}}\max_{i \in \mathcal{S}_p(x,r)} 
   \big[\hat{f}(X_i) \big]^2   dx\\
   & \geq \mathbb{E}\int_{\mathcal{X}}\max_{i \in \mathcal{S}_p(x,r)} 
   (|\xi_{i}| - \delta)_+^2dx \\
    & = \int_{\mathcal{X}}\mathbb{E}\max_{i \in \mathcal{S}_p(x,r)} (|\xi_{i}| - \delta)_+^2dx,
\end{split}
\end{equation}
where the second inequality follows from (\ref{eq:gogo_1}), and the last equality follows from Fubini's theorem. Combining (\ref{eq:adversairla_op}) with (\ref{eq:duzi1}), we obtain the minimax lower bound in 
Theorem~\ref{theo:lower}.

\section{Proof of Theorem~\ref{theo:upper}}

When deriving the upper bound in Theorem~\ref{theo:upper}, we consider the setting that $\xi_i$ are i.i.d. sub-Gaussian random variables. We decompose the adversarial $L_2$-risk of $\hat{f}_{\delta,\tau}$ into the following two terms:
\begin{equation}\label{eq:goal5}
  \begin{split}
    R_{r}(\hat{f}_{\delta,\tau}, f^*)& = \mathbb{E}\sup_{X' \in \mathcal{X} \cap 
    B_p(X,r)}\left[\hat{f}_{\delta,\tau}(X') - f^*(X)\right]^2\\
       & \leq 2\mathbb{E}\sup_{X' \in \mathcal{X} \cap B_p(X,r)}\left[\hat{f}_{\delta,\tau}(X') - 
       \hat{f}(X')\right]^2 + 2\mathbb{E}\sup_{X' \in \mathcal{X} \cap B_p(X,r)}\left[ \hat{f}(X') - 
       f^*(X)\right]^2.
  \end{split}
\end{equation}

We first upper bound the second term in (\ref{eq:goal5}). Since the density of $X$ is upper bounded by $\bar{\mu}$, to control the second term in (\ref{eq:goal5}), it 
suffices to derive an upper bound on the pointwise risk for any $x \in \mathcal{X}$. Specifically, we obtain
\begin{equation}\label{eq:yuandan2}
\begin{split}
    & \mathbb{E}\sup_{x' \in \mathcal{X} \cap B_p(x,r)}\left[ \hat{f}(x') - f^*(x)\right]^2\\
     &\leq 2\mathbb{E}\sup_{x' \in \mathcal{X} \cap B_p(x,r)}\left[ \hat{f}(x') - f^*(x')\right]^2+ 
     2\mathbb{E}\sup_{x' \in \mathcal{X} \cap B_p(x,r)}\left[ f^*(x') - f^*(x) \right]^2\\
     & \lesssim r^{2\beta} + n^{-\frac{2\beta}{2\beta+d}}+r^{2(1\wedge \beta)} \lesssim r^{2(1\wedge \beta)} + 
     n^{-\frac{2\beta}{2\beta+d}},
\end{split}
\end{equation}
where the second inequality follows from the condition (\ref{eq:local_risk}) and Lemma B.1 in \cite{Peng2025aos}.

To upper bound the first term in \eqref{eq:goal5}, it suffices to upper bound the pointwise risk 
$\mathbb{E}\sup_{x' \in \mathcal{X} \cap B_p (x,r)}[ \hat{f}_{\delta,\tau}(x') - \hat{f}(x') ]^2$ for an arbitrary $x \in \mathcal{X}$. Recall the index sets $\mathcal{S}_p(x,r) = \{ i: X_i \in \mathcal{X} \cap B_p(x,r) \}$ and 
$\mathcal{S}_p(x,r+\tau) = \{ i: X_i \in \mathcal{X} \cap B_p(x,r+\tau) \}$. By the definition of the 
$\delta$-interpolating estimator in \eqref{eq:lpe_interpolated}, we obtain
    \begin{equation}\label{eq:tusi_1}
  \begin{split}
\mathbb{E}\sup_{ x' \in \mathcal{X} \cap B_p (x,r)}\left[ \hat{f}_{\delta,\tau}(x') - \hat{f}(x') \right]^2& = 
\mathbb{E}\sup_{\substack{x' \in \mathcal{X} \cap B_p (x,r)  \cap [ \cup_{i=1}^n B_p(X_i, \tau) ] }}\left[ 
\hat{f}_{\delta,\tau}(x') - \hat{f}(x') \right]^2 \\ 
 & \leq \mathbb{E}\sup_{\substack{x' \in \mathcal{X} \cap B_p (x,r) \cap [ \cup_{i \in \mathcal{S}_p(x,r + \tau)} 
 B_p(X_i, \tau) ] }}\left[ \hat{f}_{\delta,\tau}(x') - \hat{f}(x') \right]^2 \\ 
 & = \mathbb{E}\sup_{\substack{x' \in \mathcal{X} \cap  \{ \cup_{i \in \mathcal{S}_p(x,r + \tau)} [ B_p (x,r) \cap 
 B_p(X_i, \tau)] \} }}\left[ \hat{f}_{\delta,\tau}(x') - \hat{f}(x') \right]^2\\
 & \leq \mathbb{E} \max_{i \in \mathcal{S}_p(x,r+\tau)} \left[ \left| Y_i - \hat{f}(X_i)  \right| - \delta 
 \right]_+^2. 
  \end{split}
  \end{equation}
  The first step in (\ref{eq:tusi_1}) uses the fact that $\hat{f}_{\delta,\tau}(x') = \hat{f}(x')$ whenever 
  $x' \in \mathcal{X} \setminus \cup_{i = 1}^n B_p(X_i, \tau)$. The second step in (\ref{eq:tusi_1}) follows 
  from the inclusion 
  \begin{equation}\label{eq:include}
    B_p (x,r)  \cap [ \cup_{i=1}^n B_p(X_i, \tau) ] \subseteq \cup_{i \in \mathcal{S}_p(x,r + \tau)} B_p(X_i, \tau).
  \end{equation}
  To verify (\ref{eq:include}), take any $x'$ in the left-hand side of (\ref{eq:include}). Then there exists an index $i$ such that $\|x' - x\|_p \leq r$ and $\| x' - X_i\|_p \leq \tau$. The triangle inequality implies $\|X_i - x\|_p \leq r + \tau$, hence $i \in \mathcal{S}_p(x,r + \tau)$ and $\| x' - X_i\|_p \leq \tau$. The last step in \eqref{eq:tusi_1} is a direct consequence of the definition in \eqref{eq:lpe_interpolated} and the condition (\ref{eq:lianjie}). Indeed, for any  $x' \in \mathcal{X} \cap  \{ \cup_{i \in \mathcal{S}_p(x,r + \tau)} [ B_p (x,r) \cap 
 B_p(X_i, \tau)] \}$, we have
  \begin{equation*}\label{eq:kuainian}
  \begin{split}
     \left| \hat{f}_{\delta,\tau}(x') - \hat{f}(x') \right| & \leq \max_{i \in \mathcal{S}_p(x',\tau)} \left[ \left| Y_i 
     - \hat{f}(X_i) \right| - \delta \right]_+  \leq \max_{i \in \mathcal{S}_p(x,r+\tau)} \left[ \left| Y_i - \hat{f}(X_i) \right| - 
     \delta \right]_+.
  \end{split}
\end{equation*}
Therefore, the inequality (\ref{eq:tusi_1}) is proved. 

Continuing from \eqref{eq:tusi_1}, we obtain  
      \begin{equation}\label{eq:yuandan1}
  \begin{split}
&\mathbb{E}\sup_{ x' \in \mathcal{X} \cap B_p (x,r)}\left[ \hat{f}_{\delta,\tau}(x') - \hat{f}(x') \right]^2 \leq \mathbb{E} \max_{i \in \mathcal{S}_p(x,r+\tau)}\left[ \left| Y_i - \hat{f}(X_i)  \right| - \delta 
 \right]_+^2 \\
      &\leq \mathbb{E} \max_{i \in \mathcal{S}_p(x,r+\tau)}\left[ \left| Y_i - f^*(X_i) \right| + \left|f^*(X_i) 
      -\hat{f}(X_i)  \right| - \delta \right]_+^2 \\
       & \leq 2 \mathbb{E} \max_{i \in \mathcal{S}_p(x,r+\tau)}\left( \left| Y_i - f^*(X_i) \right|  - \delta 
       \right)_+^2 + 2\mathbb{E} \max_{i \in \mathcal{S}_p(x,r+\tau)}\left[f^*(X_i) -\hat{f}(X_i)  \right]^2 \\
       & \leq 2 \mathbb{E} \max_{i \in \mathcal{S}_p(x,r+\tau)}\left( \left| \xi_i \right|  - \delta \right)_+^2 + 
       2\mathbb{E}\sup_{x' \in \mathcal{X} \cap B_p (x,r+\tau)}\left[f^*(x') -\hat{f}(x')  \right]^2\\
       & \lesssim  \mathbb{E} \max_{i \in \mathcal{S}_p(x,r+\tau)}\left( \left| \xi_i \right|  - \delta 
       \right)_+^2+ (r+\tau)^{2(1\wedge \beta)} + n^{-\frac{2\beta}{2\beta+d}},  
  \end{split}
  \end{equation}
  where the last step follows from \eqref{eq:local_risk}. Substituting \eqref{eq:yuandan2} and \eqref{eq:yuandan1} into \eqref{eq:goal5} completes the proof.

\section{Proof of Theorem~\ref{theo:low}}\label{sec:proof_3}
 
\subsection{Notation and preliminaries}

Recall that there exist constants $c>0$ and $0<\gamma\leq 1$ satisfying  
\begin{equation}\label{eq:gogo_3}
  0 \leq nr^d \leq c n^{\gamma}.
\end{equation}
Define $\mathcal{S}_p(x,r) \triangleq \{ i: X_i \in \mathcal{X} \cap B_p(x,r) \}$, $n_x \triangleq 
\mathrm{Card}(\mathcal{S}_p(x,r))$, $\bar{r} \triangleq (cn^{\gamma - 1})^{1/d}$, $\mathcal{S}_p(x,\bar{r}) \triangleq \{ i: X_i \in \mathcal{X} \cap B_p(x,\bar{r}) \}$, and $\bar{n}_x \triangleq 
\mathrm{Card}(\mathcal{S}_p(x,\bar{r}))$. 

We first derive a uniform upper bound on $\bar{n}_x$ (i.e., the cardinality of $\mathcal{S}_p(x,\bar{r})$) for all $x\in\mathcal{X}$. Note that the support $\mathcal{X}=[0,1]^d$ satisfies the regularity condition in 
\cite{Audibert2007}: there exist constants $r_{\mu}>0$ and $0<c_{\mu}<1$ such that for any $0 \leq r \leq r_{\mu}$ and $x \in \mathcal{X}$,
\begin{equation*}
  \lambda\left\{\mathcal{X} \cap  B_p(x,r) \right\} \geq c_{\mu} v_pr^d,
\end{equation*}
where $v_p$ is the volume of the unit $\ell_p$-ball in $\mathbb{R}^d$. This yields
\begin{equation}\label{eq:nianqing1}
  \mathbb{P}[X \in \mathcal{X} \cap B_p(x,r)] \geq \barbelow{\mu}\lambda[\mathcal{X} \cap B_p(x,r)] \geq 
  \barbelow{\mu}c_{\mu} v_pr^d   \triangleq c_1 r^d,
\end{equation}
where $0<c_1<1$. Define $c_2 \triangleq 2^{d+1} \bar{\mu} v_p$ and the event 
\begin{equation}\label{eq:gogo_event_1}
  \mathcal{E}_1 \triangleq \left\{ \forall x \in \mathcal{X}, \bar{n}_x \leq c_2 n\bar{r}^d \right\}.  
\end{equation}
Then we have the following Lemma. 
\begin{lemma}\label{lem:basic1}
  Suppose Assumption~\ref{ass:bounded_density} holds. Then there must exist two positive constant $C$ and $c_3$ such that 
  \begin{equation}\label{eq:gogo_key_pro}
  \mathbb{P}\left( \mathcal{E}_1 \right) \geq 1 - C n^{1-\gamma} \exp\left( -c_3n^{\gamma} \right).   
\end{equation}
\end{lemma}

\begin{proof}[Proof of Lemma~\ref{lem:basic1}]

Note that for a given $x \in \mathcal{X}$, we have
\begin{equation}\label{eq:gogo_2}
\begin{split}
\mathbb{P}\left( \sum_{i=1}^{n}1_{\{ X_i \in \mathcal{X} \cap B_p(x,2\bar{r}) \}} > c_2 n\bar{r}^d\right) & \leq  
\mathbb{P}\left( \sum_{i=1}^{n}1_{\{ X_i \in \mathcal{X} \cap B_p(x,2\bar{r}) \}} > 2n\mathbb{P}[X_i \in 
\mathcal{X} \cap B_p(x,2\bar{r})] \right) \\
     & \leq \exp\left( -\frac{n\mathbb{P}[X_i \in \mathcal{X} \cap B_p(x,2\bar{r})]}{3} \right)\\
     & \leq  \exp\left( -\frac{c_12^d}{3}n\bar{r}^d \right) = \exp\left( -\frac{c_1c2^d}{3}n^{\gamma} \right) 
     \triangleq \exp\left( -c_3n^{\gamma} \right),
\end{split}
\end{equation}
where the first inequality follows from 
\begin{equation*}
  2\mathbb{P}[X_i \in \mathcal{X} \cap B_p(x,2\bar{r})] \leq 2\bar{\mu} v_p (2\bar{r})^d = 2^{d+1} \bar{\mu} v_p 
  \bar{r}^d = c_2 \bar{r}^d,
\end{equation*}
the second inequality is derived from the multiplicative Chernoff bound, the last inequality follows from 
(\ref{eq:nianqing1}), and the first equality follows from the definition of $\bar{r}$. 

Let $\{ z_1,\ldots, z_{M} \}$ be an $\bar{r}$-covering of $\mathcal{X}$ under the $\ell_p$-metric, so that $M \leq C\bar{r}^{-d}$ for some constant $C>0$. For any $x \in \mathcal{X}$, there exists a $z_i$ with $\| x - z_i \|_p \leq \bar{r}$. The triangle inequality gives $B_p(x , \bar{r}) \subseteq B_p(z_i , 2\bar{r})$; hence, 
\begin{equation*}
  \begin{split}
     \mathbb{P}\left( \mathcal{E}_1 \right) & = \mathbb{P}\left( \forall x \in \mathcal{X}, \sum_{i=1}^{n}1_{\{ 
     X_i \in \mathcal{X} \cap B_p(x,\bar{r}) \}} \leq c_2 n\bar{r}^d \right)\\
       & \geq \mathbb{P}\left( \forall k \in \{1,\ldots,M \}, \sum_{i=1}^{n}1_{\{ X_i \in \mathcal{X} \cap 
       B_p(z_k,2\bar{r}) \}} \leq c_2 n\bar{r}^d \right)\\
       & = 1 - \mathbb{P}\left( \exists k \in \{1,\ldots,M \}, \sum_{i=1}^{n}1_{\{ X_i \in \mathcal{X} \cap 
       B_p(z_k,2\bar{r}) \}} > c_2 n\bar{r}^d \right)\\
       & \geq 1 - M \mathbb{P}\left( \sum_{i=1}^{n}1_{\{ X_i \in \mathcal{X} \cap B_p(z_k,2\bar{r}) \}} > c_2 
       n\bar{r}^d \right)\\
       & \geq 1 - C\bar{r}^{-d} \exp\left( -c_3n^{\gamma} \right),
  \end{split}
\end{equation*}
where the first equality uses the definition of $\bar n_x$, the first inequality follows from
$B_p(x,\bar r)\subseteq B_p(z_i,2\bar r)$, and the final inequality uses (\ref{eq:gogo_2}). Since $\bar{r}^{-d} = c^{-1}n^{1-\gamma}$, the above bound simplifies to (\ref{eq:gogo_key_pro}) after redefining the constant $C>0$.

\end{proof}

\subsection{Finalizing the proof of Theorem~\ref{theo:low}}

Since $r \leq \bar{r}$ for any $r$ satisfying (\ref{eq:gogo_3}), we have 
$\mathcal{S}_p(x,r) \subseteq \mathcal{S}_p(x,\bar{r})$ for all $x \in \mathcal{X}$. Consequently, 
\begin{equation*}
 \int_{\mathcal{X}}\mathbb{E}\max_{i \in \mathcal{S}_p(x,r)} (|\xi_{i}| - \delta)_+^2dx  \leq 
 \int_{\mathcal{X}}\mathbb{E}\max_{i \in \mathcal{S}_p(x,\bar{r})} (|\xi_{i}| - \delta)_+^2dx
\end{equation*}
for every $r$ satisfying (\ref{eq:gogo_3}). Thus, the main task is to derive an upper bound for $\mathbb{E}\max_{i 
\in \mathcal{S}_p(x,\bar{r})} (|\xi_{i}| - \delta)_+^2$ that holds uniformly over $x\in\mathcal{X}$. For any $x \in \mathcal{X}$, we have 
\begin{subequations}
\begin{align}
   \mathbb{E}\max_{i \in \mathcal{S}_p(x,\bar{r})} (|\xi_{i}| - \delta)_+^2 &  =\mathbb{E}_{X^n}\mathbb{E}_{Y^n 
   \mid X^n}\max_{1 \leq i \leq \bar{n}_x} (|\xi_{i}| - \delta)_+^2 \nonumber \\
     & =  \mathbb{E}_{X^n}\mathbb{E}_{Y^n \mid X^n}\max_{1 \leq i \leq \bar{n}_x} (|\xi_{i}| - \delta)_+^2 1_{\{ 
   \mathcal{E}_1 \}} \label{eqn:line-1} \\
     & + \mathbb{E}_{X^n}\mathbb{E}_{Y^n \mid X^n}\max_{1 \leq i \leq \bar{n}_x} (|\xi_{i}| - 
   \delta)_+^2 1_{\{ \mathcal{E}_1^c \}}. \label{eqn:line-2}
\end{align}
\end{subequations}

\subsubsection{Bounding the term in (\ref{eqn:line-1})}\label{sec:proof_aaa}

For the term in (\ref{eqn:line-1}), we obtain 
 \begin{equation*}
 \begin{split}
      & \mathbb{E}_{X^n}\mathbb{E}_{Y^n \mid X^n}\max_{1 \leq i \leq \bar{n}_x} (|\xi_{i}| - \delta)_+^2 1_{\{ 
      \mathcal{E}_1 \}} \leq \mathbb{E}_{X^n}\mathbb{E}_{Y^n \mid X^n}\max_{1 \leq i \leq c_2 n\bar{r}^d} 
      (|\xi_{i}| - \delta)_+^2 \\
      & = \mathbb{E}\max_{1 \leq i \leq c_2 n\bar{r}^d} (|\xi_{i}| - \delta)_+^2 = \mathbb{E}\max_{1 \leq i \leq 
      c_4n^{\gamma}} (|\xi_{i}| - \delta)_+^2, 
 \end{split} 
 \end{equation*}
 where the inequality follows from the definition of $\mathcal{E}_1$ in (\ref{eq:gogo_event_1}), the first 
 equality removes the conditioning, and the last equality uses the definition of $\bar{r}$, with $c_4 \triangleq 
 cc_2$. 

 Define $t_0\triangleq(\sqrt{2\nu^2(\frac{2\beta}{2\beta+d}+\gamma) \log n}-\delta)_{+}^2$, where the positive 
 constant $\nu$ characterizes the sub-Gaussian tail behavior of $\xi_i$ and will be specified later. By integrating 
 the tail probability,
\begin{equation}\label{eq:yanse1}
\begin{split}
  & \mathbb{E}\max_{1 \leq i \leq c_4n^{\gamma}} (|\xi_i| - \delta)_+^2 = \int_{0}^{\infty} \mathbb{P} \left\{  
   \max_{1 \leq i \leq c_4n^{\gamma}} (|\xi_i| - \delta)_+^2>t \right\}dt \\
     & = \int_{0}^{t_0} \mathbb{P} \left\{  \max_{1 \leq i \leq c_4n^{\gamma}} (|\xi_i| - \delta)_+^2>t \right\}dt 
     + \int_{t_0}^{\infty} \mathbb{P} \left\{  \max_{1 \leq i \leq c_4n^{\gamma}} (|\xi_i| - \delta)_+^2>t 
     \right\}dt\\
     & \leq  \left[\sqrt{2\nu^2\left(\frac{2\beta}{2\beta+d}+\gamma\right) \log n}-\delta\right]_{+}^2 + 
     \int_{t_0}^{\infty} \mathbb{P} \left\{  \max_{1 \leq i \leq c_4n^{\gamma}} (|\xi_i| - \delta)_+^2>t 
     \right\}dt.
\end{split}
\end{equation}
To handle the second term in (\ref{eq:yanse1}), we use the sub-Gaussianity of $\xi_i$. By Theorem 2.6 of 
\cite{Wainwright_2019}, there exist a constant $C$ and a Gaussian random variable $Z\sim N(0,\nu^{2})$ such that 
$\mathbb{P}(|\xi_i|>s) \leq C\mathbb{P}(|Z|>s)$ for all $s\geq 0$. Applying the union bound, for any $t\ge t_0$ we obtain 
\begin{equation*}
  \begin{split}
     \mathbb{P} \left\{  \max_{1 \leq i \leq c_4n^{\gamma}} (|\xi_i| - \delta)_+^2>t \right\} &  =  
     c_4n^{\gamma}\mathbb{P} \left\{  \left| \xi_i \right| > \delta+\sqrt{t} \right\} \leq 
     c_4Cn^{\gamma}\mathbb{P} \left\{  \left| Z \right| > \delta+\sqrt{t} \right\}\\
     & = 2 c_4Cn^{\gamma} \tilde{\Phi}\left( \frac{\delta+\sqrt{t}}{\nu} \right) \lesssim n^{\gamma} \exp\left( 
     -\frac{(\delta+\sqrt{t})^2}{2\nu^2}  \right),
  \end{split}
\end{equation*}
where the last step is derived by combining the fact that 
$$
\frac{\delta+\sqrt{t}}{\nu} \geq \frac{\delta+\sqrt{t_0}}{\nu} = 
\frac{\delta+(\sqrt{2\nu^2(\frac{2\beta}{2\beta+d}+\gamma) \log n}-\delta)_{+}}{\nu} \geq 
\sqrt{2\left(\frac{2\beta}{2\beta+d}+\gamma\right) \log n} \geq C
$$
when $n>2$ and the bound for Mills ratio $\tilde{\Phi}(x) \leq C'\exp(-x^2/2)$ when $x > C$ for a constant $C'$. 
Therefore, the second term in (\ref{eq:yanse1}) can be upper bounded by 
\begin{equation}\label{eq:yanse2}
  \begin{split}
     \int_{t_0}^{\infty} \mathbb{P} \left\{  \max_{1 \leq i \leq c_4n^{\gamma}} (|\xi_i| - \delta)_+^2>t 
     \right\}dt & \lesssim n^{\gamma} \int_{t_0}^{\infty}\exp\left( -\frac{(\delta+\sqrt{t})^2}{2\nu^2}  \right)dt 
     \\
       & = 2n^{\gamma} \int_{(\sqrt{2\nu^2(\frac{2\beta}{2\beta+d}+\gamma) \log 
       n}-\delta)_{+}}^{\infty}t\exp\left( -\frac{(\delta+t)^2}{2\nu^2} \right)dt\\
       & = 2\nu n^{\gamma} \int_{\frac{(\sqrt{2\nu^2(\frac{2\beta}{2\beta+d}+\gamma) \log 
       n}-\delta)_{+}+\delta}{\nu}}^{\infty}(\nu t - \delta)\exp\left( -\frac{t^2}{2} \right)dt\\
       & \leq 2\nu^2 n^{\gamma} \int_{\frac{(\sqrt{2\nu^2(\frac{2\beta}{2\beta+d}+\gamma) \log 
       n}-\delta)_{+}+\delta}{\nu}}^{\infty}t \exp\left( -\frac{t^2}{2} \right)dt\\
       & = 2\nu^2 n^{\gamma} \exp\left[ -\frac{\left[ (\sqrt{2\nu^2(\frac{2\beta}{2\beta+d}+\gamma) \log 
       n}-\delta)_{+} +\delta \right]^2}{2\nu^2}  \right],\\
  \end{split}
\end{equation}
where the last steps follows from the fact that $\int_x^{\infty} v \exp\left(-v^2 / 2\right) d v=\exp(-x^2 / 2) \quad \text { for } x>0$.

Combining (\ref{eq:yanse1})--(\ref{eq:yanse2}), we have shown when $\delta > 
\sqrt{2\nu^2(\frac{2\beta}{2\beta+d}+\gamma) \log n}$,
\begin{equation*}
\begin{split}
  & \mathbb{E}\max_{1 \leq i \leq c_4n^{\gamma}} (|\xi_i| - \delta)_+^2\\
   &\lesssim \left[\sqrt{2\nu^2\left(\frac{2\beta}{2\beta+d}+\gamma\right) \log n}-\delta\right]_{+}^2+ n^{\gamma} 
   \exp\left[ -\frac{\left[ (\sqrt{2\nu^2(\frac{2\beta}{2\beta+d}+\gamma) \log n}-\delta)_{+} +\delta 
   \right]^2}{2\nu^2}  \right] \\
   & = n^{\gamma} \exp\left( -\frac{\delta^2}{2\nu^2} \right) \lesssim  
   n^{-(\frac{2\beta}{2\beta+d}+\gamma-\gamma)} = O\left( n^{-\frac{2\beta}{2\beta+d}} \right).
\end{split}
\end{equation*}

\subsubsection{Bounding the term in (\ref{eqn:line-2})}

We bound the second term in (\ref{eqn:line-2}) as follows:  
\begin{equation*}
\begin{split}
     & \mathbb{E}_{X^n}\mathbb{E}_{Y^n \mid X^n}\max_{1 \leq i \leq \bar{n}_x} (|\xi_{i}| - \delta)_+^2 1_{\{ 
     \mathcal{E}_1^c \}} \leq \mathbb{E}_{X^n}\mathbb{E}_{Y^n \mid X^n}\max_{1 \leq i \leq n} (|\xi_{i}| - 
     \delta)_+^2 1_{\{ \mathcal{E}_1^c \}} \\
     & = \mathbb{E}_{X^n}1_{\{ \mathcal{E}_1^c \}}\mathbb{E}_{Y^n \mid X^n}\max_{1 \leq i \leq n} (|\xi_{i}| - 
     \delta)_+^2 = \mathbb{P}(\mathcal{E}_1^c)\mathbb{E}\max_{1 \leq i \leq n} (|\xi_{i}| - \delta)_+^2\\
     & \leq \mathbb{P}(\mathcal{E}_1^c)\mathbb{E}\max_{1 \leq i \leq n}|\xi_{i}|^2. 
\end{split}
\end{equation*}
For sub-Gaussian errors, the standard bound $\mathbb{E}\max_{1 \leq i \leq n}|\xi_{i}|^2 \lesssim \log n$ 
holds. Combining this with the result in \eqref{eq:gogo_key_pro}, we obtain 
\begin{equation*}
  \mathbb{E}_{X^n}\mathbb{E}_{Y^n \mid X^n}\max_{1 \leq i \leq \bar{n}_x} (|\xi_{i}| - \delta)_+^2 1_{\{ 
  \mathcal{E}_1^c \}} \lesssim (\log n)n^{1-\gamma} \exp\left( -c_3n^{\gamma} \right)  = O\left( 
  n^{-\frac{2\beta}{2\beta+d}} \right). 
\end{equation*}
This completes the proof of Theorem~\ref{theo:low}.

\section{Proof of Theorem~\ref{theo:Moderate}}

Recall that $n_x = \mathrm{Card}(\mathcal{S}_p(x,r))$ denotes the number of observations in $\mathcal{S}_p(x,r)$. 
Our primary task is to obtain a lower bound for $\int_{\mathcal{X}}\mathbb{E}\max_{i \in \mathcal{S}_p(x,r)} 
(|\xi_{i}| - \delta)_+^2dx$. It suffices to derive a pointwise lower bound for $\mathbb{E}\max_{i \in 
\mathcal{S}_p(x,r)} (|\xi_{i}| - \delta)_+^2$ that holds for all $x \in \mathcal{X}$. For any $x \in \mathcal{X}$, we have 
\begin{equation}\label{eq:nianqing2}
\begin{split}
   \mathbb{E}\max_{i \in \mathcal{S}_p(x,r)} (|\xi_{i}| - \delta)_+^2 & = \mathbb{E}_{X^n}\mathbb{E}_{Y^n \mid 
   X^n}\max_{i \in \mathcal{S}_p(x,r)} (|\xi_{i}| - \delta)_+^2\\
     & = \mathbb{P}_{X^n}(n_x = 0) \times 0 + \mathbb{P}_{X^n}(n_x = 1) \mathbb{E}(|\xi_{i}| - \delta)_+^2 \\
     &\quad+ \sum_{k=2}^{n}\mathbb{P}_{X^n}(n_x = k)\mathbb{E}\max_{1 \leq i \leq k} (|\xi_{i}| - \delta)_+^2\\
     & \geq \mathbb{P}_{X^n}(n_x = 1) \mathbb{E}(|\xi_{i}| - \delta)_+^2.
\end{split}
\end{equation}

We first consider the case where $nr^d \to 0$. Based on (\ref{eq:nianqing1}), we know $\mathbb{P}[X \in \mathcal{X} \cap B_p(x,r)] \geq  c_1 r^d$ for any $x$. Consequently, when $nr^d \to 0$ and for any $x \in \mathcal{X}$,
\begin{equation}\label{eq:nianqing3}
  \mathbb{P}_{X^n}(n_x = 1) \geq n (c_1 r^d )(1-c_1 r^d)^{n-1} \sim c_1nr^d \exp[-c_1(n-1)r^d] \asymp nr^d.
\end{equation}
Using technique on the risk of soft-thresholding estimators at zero mean \citep[see, e.g.,][]{Donoho1994}, we obtain
\begin{equation}\label{eq:soft1}
\begin{split}
   \mathbb{E}(|\xi_{i}| - \delta)_+^2 & = 2 \int_{\delta}^{\infty}(y - \delta)^2 \frac{1}{\sqrt{2\pi 
   \sigma^2}}\exp\left( - \frac{y^2}{2\sigma^2} \right)dy\\
     & = 2\sigma^2 \int_{\delta/\sigma}^{\infty}(t - \delta/\sigma)^2 \phi(t)dt \\
     & = 2(\delta+\sigma)^2 \tilde{\Phi}(\delta/\sigma) - 2\sigma\delta\phi(\delta/\sigma), \\
\end{split}
\end{equation}
where $\phi(\cdot)$ and $\tilde{\Phi}$ are density function and tail probability of the standard Gaussian 
distribution. Recall the standard Mills-type bound that for any $x>0$, $\tilde{\Phi}(x)> \frac{x}{x^2+1}\phi(x)$. Applying this lower bound for $\tilde{\Phi}$, we obtain 
\begin{equation*}
  \begin{split}
       & 2(\delta+\sigma)^2 \tilde{\Phi}(\delta/\sigma) - 2\sigma\delta\phi(\delta/\sigma)  \geq 
       2(\delta+\sigma)^2 \frac{\delta \sigma}{\delta^2+ \sigma^2}\phi(\delta/\sigma)- 
       2\sigma\delta\phi(\delta/\sigma) = \frac{4\delta^2 \sigma^2}{\delta^2+ \sigma^2}\phi(\delta/\sigma). 
  \end{split}
\end{equation*}
Hence, as $\delta \to \infty$, $\frac{4\delta^2 \sigma^2}{\delta^2+ \sigma^2}$ remains bounded below by a positive 
constant. It follows that 
\begin{equation}\label{eq:gogo_zem}
  \mathbb{E}(|\xi_{i}| - \delta)_+^2 \gtrsim \exp\left( - \frac{\delta^2}{2\sigma^2} \right) \gtrsim \exp\left( - 
  \frac{C_2^2(\log n)^{2t}}{2\sigma^2} \right), 
\end{equation}
where the second inequality follows from the condition $\delta \leq C_2(\log n)^{t}$. Consequently, when $nr^d \to 
0$,
\begin{equation*}
  \mathbb{E}\max_{i \in \mathcal{S}_p(x,r)} (|\xi_{i}| - \delta)_+^2 \gtrsim nr^d\exp\left( - \frac{C_2^2(\log 
  n)^{2t}}{2\sigma^2} \right)
\end{equation*}
for any $x \in \mathcal{X}$. 

Next, we consider the case where $nr^d \gtrsim 1$ and $r \to 0$. For any $x \in \mathcal{X}$, we have
\begin{equation}\label{eq:nianqing5}
\begin{split}
   \mathbb{E}\max_{i \in \mathcal{S}_p(x,r)} (|\xi_{i}| - \delta)_+^2 &  = \mathbb{P}_{X^n}(n_x = 1) 
   \mathbb{E}(|\xi_{i}| - \delta)_+^2 + \sum_{k=2}^{n}\mathbb{P}_{X^n}(n_x = k)\mathbb{E}\max_{1 \leq i \leq k} 
   (|\xi_{i}| - \delta)_+^2\\
     & \geq \mathbb{P}_{X^n}(n_x = 1) \mathbb{E}(|\xi_{i}| - \delta)_+^2+ \sum_{k=2}^{n}\mathbb{P}_{X^n}(n_x = 
     k)\mathbb{E}(|\xi_{i}| - \delta)_+^2\\
     & = \mathbb{P}_{X^n}(n_x \geq 1) \mathbb{E}(|\xi_{i}| - \delta)_+^2.
\end{split}
\end{equation}
For the probability term, we obtain
\begin{equation}\label{eq:gogo_zem_1}
  \begin{split}
     \mathbb{P}_{X^n}(n_x \geq 1) & = 1 - \mathbb{P}_{X^n}(n_x =0)  = 1 - \left[ 1 - \mathbb{P}(X \in \mathcal{X} 
     \cap B_p(x,r)) \right]^n\\
       & \geq 1 - ( 1 - c_1 r^d )^n = 1- (1 - c_1 r^d)^{-\frac{1}{c_1r^d}(-c_1nr^d)} \sim 1 - \exp( -c_1nr^d ) 
       \gtrsim 1.
  \end{split}
\end{equation}
Combining (\ref{eq:gogo_zem}), (\ref{eq:nianqing5}), and (\ref{eq:gogo_zem_1}), we obtain, for any $x \in 
\mathcal{X}$, 
\begin{equation*}
  \mathbb{E}\max_{i \in \mathcal{S}_p(x,r)} (|\xi_{i}| - \delta)_+^2 \gtrsim \exp\left( - \frac{C_2^2(\log 
  n)^{2t}}{2\sigma^2} \right). 
\end{equation*}

In the case where $r$ does not converge, the lower bound in (\ref{eq:adversairla_op}) already leads to a 
non-converging lower bound. This completes the proof.

\section{Proof of Theorem~\ref{theo:high}}\label{sec:proof_5}

We proceed by considering three different scenarios of $r$.

\subsection{Low-magnitude attack}\label{sec:proof_yuandan}

In this scenario, we consider the regime $nr^d \to 0$. To establish the rate in this case, we construct matching lower and upper bounds for $\mathbb{E}\max_{i \in \mathcal{S}_p(x,r)} (|\xi_{i}| - \delta)_+^2$ uniformly over $x \in \mathcal{X}$.

\subsubsection{Lower bound}

Based on (\ref{eq:nianqing2}), we have for any $x \in \mathcal{X}$, 
\begin{equation}\label{eq:hanjia22}
\begin{split}
   \mathbb{E}\max_{i \in \mathcal{S}_p(x,r)} (|\xi_{i}| - \delta)_+^2 \geq \mathbb{P}_{X^n}(n_x = 1) 
   \mathbb{E}(|\xi_{i}| - \delta)_+^2.
\end{split}
\end{equation}
Based on Stein's identity for the soft-thresholding estimators \citep[see (2.74) in][]{johnstone2017}, we have
\begin{equation}\label{eq:wumain_1}
\begin{split}
   \mathbb{E}(|\xi_{i}| - \delta)_+^2 & = \mathbb{E}\left[ 1 - 2 \times 1_{\{ |\xi_{i}| \leq \delta \}} + 
   \min(\xi_{i}^2, \delta^2) \right] \\
     & \geq 1 - 2\mathbb{P}\left( |\xi_{i}| \leq \delta \right) = 1 - 2\mathbb{P}\left( \frac{|\xi_{i}|}{\sigma} 
     \leq \frac{\delta}{\sigma} \right)\\
     & = 1 - 2 \times \left[ 1 - 2\tilde{\Phi}\left( \frac{\delta}{\sigma} \right) \right] = 4\tilde{\Phi}\left( 
     \frac{\delta}{\sigma} \right) - 1 > c,
\end{split}
\end{equation}
where the last inequality follows from $0 \leq \delta < \sigma \tilde{\Phi}^{-1}(\frac{1+c}{4})\triangleq C_3 
\sigma$ for any $0 < c< 1$, and $\tilde{\Phi}$ is the tail probability of the standard Gaussian distribution. From (\ref{eq:nianqing3}), we see that for any $x \in \mathcal{X}$, $\mathbb{P}_{X^n}(n_x = 1) \gtrsim nr^d$. Combining this with (\ref{eq:hanjia22})--(\ref{eq:wumain_1}), we establish the lower bound
\begin{equation*}
   \mathbb{E}\max_{i \in \mathcal{S}_p(x,r)} (|\xi_{i}| - \delta)_+^2 \gtrsim nr^d
\end{equation*}
for all $x \in \mathcal{X}$.

\subsubsection{Upper bound}

Similar to (\ref{eq:nianqing2}), we have the following expression of $\mathbb{E}\max_{i \in \mathcal{S}_p(x,r)} (|\xi_{i}| - \delta)_+^2$ for all $x \in \mathcal{X}$: 
  \begin{equation}\label{eq:wumai_3}
    \mathbb{E}\max_{i \in \mathcal{S}_p(x,r)} (|\xi_{i}| - \delta)_+^2 = \mathbb{P}_{X^n}(n_x = 1) 
    \mathbb{E}(|\xi_{1}| - \delta)_+^2 + \sum_{k=2}^{n}\mathbb{P}_{X^n}(n_x = k)\mathbb{E}\max_{1 \leq i \leq k} 
    (|\xi_{i}| - \delta)_+^2.
  \end{equation}
Recall the results in (\ref{eq:soft1}), we obtain $
   \mathbb{E}(|\xi_{1}| - \delta)_+^2  = 2(\delta+\sigma)^2 \tilde{\Phi}(\delta/\sigma) - 
   2\sigma\delta\phi(\delta/\sigma)$. 
When $0 \leq \delta < \sigma \tilde{\Phi}^{-1}(\frac{1+c}{4})$, we have
\begin{equation*}
  2(\delta+\sigma)^2 \tilde{\Phi}(\delta/\sigma) - 2\sigma\delta\phi(\delta/\sigma) \lesssim 
  \tilde{\Phi}(\delta/\sigma) \leq \frac{1}{2}.
\end{equation*}
In addition, we need to establish an upper bound for $\mathbb{E}\max_{1 \leq i \leq k} (|\xi_{i}| - \delta)_+^2$ 
for any $k$, where $\xi_i$ i.i.d. $N(0,\sigma^2)$. Using the similar technique in Section~\ref{sec:proof_aaa}, we have
  \begin{equation}\label{eq:bushuo1}
  \begin{split}
     \mathbb{E}\max_{1 \leq i \leq k} (|\xi_i| - \delta)_+^2 & \lesssim  (\sqrt{2 \sigma^2\log (2 
     k)}-\delta)_{+}^2  + 2\sigma^2 k \exp\left[ -\frac{\left[ (\sqrt{2\sigma^2 \log(2k)} -\delta)_+ +\delta 
     \right]^2}{2\sigma^2}  \right] \\
       & \leq (\sqrt{2 \sigma^2\log (2 k)}-\delta)_{+}^2  + \sigma^2.
  \end{split}
  \end{equation}

Now we turn to an upper bound for the probability terms in (\ref{eq:wumai_3}). For any $x \in \mathcal{X}$, let $p = \mathbb{P}\left[ X \in \mathcal{X} \cap B_p(x,r) \right]$. By the upper bound on the density, we have $p \leq \bar{\mu} v_p r^d = c_3 r^d$ for any $x \in \mathcal{X}$. Therefore, we have 
  \begin{equation*}
    \begin{split}
       \mathbb{P}_{X^n}(n_x = 1) & \leq n (c_3 r^d) (1 - c_3 r^d)^{n-1}  \sim c_3 n r^d \exp\left( - c_3nr^d 
       \right) \lesssim nr^d
    \end{split}
  \end{equation*}
  when $nr^d \to 0$. Combining the above results, we have proved that
  \begin{equation}\label{eq:fansi_1}
    \mathbb{E}\max_{i \in \mathcal{S}_p(x,r)} (|\xi_{i}| - \delta)_+^2 \lesssim nr^d+ \sum_{k=2}^n\binom{n}{k} 
    p^k(1-p)^{n-k} \left[(\sqrt{2 \sigma^2\log (2 k)}-\delta)_{+}^2  + \sigma^2 \right].
  \end{equation}
  Next we bound the summation in (\ref{eq:fansi_1}). Using $\binom{n}{k} \leq \frac{n^k}{k!}$, we get 
$$
\binom{n}{k} p^k(1-p)^{n-k} \leq \frac{n^k p^k}{k!}(1-p)^{n-k} \leq \frac{(np)^k}{k!} \lesssim 
\frac{(nr^d)^k}{k!},  
$$
uniformly in $x$. Therefore, we have
\begin{equation*}
  \begin{split}
       \sum_{k=2}^n\binom{n}{k} p^k(1-p)^{n-k} \left[(\sqrt{2 \sigma^2\log (2 k)}-\delta)_{+}^2  + \sigma^2 
       \right]&  \lesssim \sum_{k=2}^n \frac{(nr^d)^k}{k!} \left[(\sqrt{2 \sigma^2\log (2 k)}-\delta)_{+}^2  + \sigma^2 
       \right] \\
       & \lesssim (nr^d)^2 \sum_{k=2}^{n} \frac{\left[(\sqrt{2 \sigma^2\log (2 k)}-\delta)_{+}^2  + \sigma^2 
       \right]}{k!}\\
       & \lesssim (nr^d)^2, 
  \end{split}
\end{equation*}
where the last step follows from the fact that $[(\sqrt{2 \sigma^2\log (2 k)}-\delta)_{+}^2  + \sigma^2 ]$ grows at most logarithmically in $k$. Putting the pieces together we obtain, uniformly in $x \in \mathcal{X}$,
  \begin{equation*}
    \mathbb{E}\max_{i \in \mathcal{S}_p(x,r)} (|\xi_{i}| - \delta)_+^2 \lesssim nr^d
  \end{equation*}
  when $nr^d \to 0$

\subsection{Moderate-magnitude attack}
Next, we consider the case where $1 \lesssim nr^d \lesssim \log n $. In this regime, $r$ still tends to zero. From 
(\ref{eq:nianqing5}), we have for any $x \in \mathcal{X}$, 
\begin{equation*}\label{eq:nianqing55}
\begin{split}
   \mathbb{E}\max_{i \in \mathcal{S}_p(x,r)} (|\xi_{i}| - \delta)_+^2 \geq  \mathbb{P}_{X^n}(n_x \geq 1) 
   \mathbb{E}(|\xi_{i}| - \delta)_+^2.
\end{split}
\end{equation*}
The closed-form expression for $\mathbb{E}(|\xi_{i}| - \delta)_+^2$ has been derived in \eqref{eq:soft1}. For the 
probability term, from (\ref{eq:gogo_zem_1}), we obtain $\mathbb{P}_{X^n}(n_x \geq 1)  \gtrsim 1$. Therefore, for 
all $x\in \mathcal{X}$, we establish the lower bound
\begin{equation*}
  \mathbb{E}\max_{i \in \mathcal{S}_p(x,r)} (|\xi_{i}| - \delta)_+^2 \gtrsim \left[ (\delta+\sigma)^2 
  \tilde{\Phi}(\delta/\sigma) - \sigma\delta\phi(\delta/\sigma) \right] > c,
\end{equation*}
where the last inequality follows from (\ref{eq:wumain_1}).

\subsection{High-magnitude attack}

In this subsection, we determine the order of the term $\int_{\mathcal{X}}\mathbb{E}\max_{i \in \mathcal{S}_p(x,r)} (|\xi_{i}| - \delta)_+^2dx$ in the regime $nr^d \geq C_4 \log n$ for some constant $C_4>0$. The approach is to derive uniform lower and upper bounds for $\mathbb{E}\max_{i \in \mathcal{S}_p(x,r)} (|\xi_{i}| - \delta)_+^2$ over all $x \in \mathcal{X}$, and to show that these bounds are of the same order. We begin by establishing the lower bound. 

\subsubsection{Lower bound}

We consider the case where $nr^d \geq \frac{4}{c_2}\log n $, where $c_2  \triangleq 
\frac{\barbelow{\mu}c_{\mu}v_p}{2^{d+1}}$. Define the event $X_1,\ldots, X_n$ that
\begin{equation*}
  \mathcal{E}_2 \triangleq \left\{ \forall x \in \mathcal{X}, n_x \geq c_2 nr^d \right\}.
\end{equation*}
Our first goal is to establish a lower bound for the probability of $\mathcal{E}_2$.

For a given $x \in \mathcal{X}$, we have
\begin{equation}\label{eq:nianqing6}
\begin{split}
\mathbb{P}\left( \sum_{i=1}^{n}1_{\{ X_i \in \mathcal{X} \cap B_p(x,\frac{r}{2}) \}} < c_2 nr^d\right) & \leq  
\mathbb{P}\left( \sum_{i=1}^{n}1_{\{ X_i \in \mathcal{X} \cap B_p(x,\frac{r}{2}) \}} < \frac{n\mathbb{P}[X_i \in 
\mathcal{X} \cap B_p(x,\frac{r}{2})]}{2} \right) \\
     & \leq \exp\left( -\frac{n\mathbb{P}[X_i \in \mathcal{X} \cap B_p(x,\frac{r}{2})]}{8} \right)\\
     & \leq \exp\left( -\frac{\barbelow{\mu}c_{\mu}v_p}{2^{d+3}}nr^d \right) = \exp\left( -\frac{c_2}{4}nr^d 
     \right),
\end{split}
\end{equation}
where the first and third inequalities follow from \eqref{eq:nianqing1}, and the second inequality is derived from the multiplicative Chernoff bound. Next, let $\{ z_1,\ldots, z_{M} \}$ be an $\frac{r}{2}$-covering set of $\mathcal{X}$ under the $\ell_p$-metric, 
where $M$ is upper bounded by $Cr^{-d}$ for some constant $C>0$. For any $x \in \mathcal{X}$, there exists a $z_i$ 
such that $\| x - z_i \|_p \leq \frac{r}{2}$. By the triangle inequality, we have $B_p(z_i , \frac{r}{2})  
\subseteq B_p(x , r)$. Therefore, we have
\begin{equation*}
  \begin{split}
     \mathbb{P}\left( \mathcal{E}_2 \right) & = \mathbb{P}\left( \forall x \in \mathcal{X}, \sum_{i=1}^{n}1_{\{ 
     X_i \in \mathcal{X} \cap B_p(x,r) \}} \geq c_2 nr^d \right)\\
       & = \mathbb{P}\left( \forall k \in \{1,\ldots,M \}, \sum_{i=1}^{n}1_{\{ X_i \in \mathcal{X} \cap 
       B_p(z_k,\frac{r}{2}) \}} \geq c_2 nr^d \right)\\
       & = 1 - \mathbb{P}\left( \exists k \in \{1,\ldots,M \}, \sum_{i=1}^{n}1_{\{ X_i \in \mathcal{X} \cap 
       B_p(z_k,\frac{r}{2}) \}} < c_2 nr^d \right)\\
       & \geq 1 - M \mathbb{P}\left( \sum_{i=1}^{n}1_{\{ X_i \in \mathcal{X} \cap B_p(z_k,\frac{r}{2}) \}} < c_2 
       nr^d \right)\\
       & \geq 1 - Cr^{-d} \exp\left( -\frac{c_2}{4}nr^d \right),
  \end{split}
\end{equation*}
where the last inequality follows from \eqref{eq:nianqing6}. When $nr^d > \frac{ 4}{c_2}\log n$, we obtain
\begin{equation*}
  Cr^{-d} \exp\left( -\frac{c_2}{4}nr^d \right) \leq \frac{C}{nr^d} \to 0.
\end{equation*}
Thus, we conclude that $\mathbb{P}\left( \mathcal{E}_2 \right) \to 1$.

We are now prepared to establish a uniform lower bound for $\mathbb{E}\max_{i \in \mathcal{S}_p(x,r)} (|\xi_{i}| - 
\delta)_+^2$ for all $x\in \mathcal{X}$. We have for all $x \in \mathcal{X}$,
\begin{equation*}
\begin{split}
   \mathbb{E}\max_{i \in \mathcal{S}_p(x,r)} (|\xi_{i}| - \delta)_+^2 & =\mathbb{E}_{X^n}\mathbb{E}_{Y^n \mid 
   X^n}\max_{1 \leq i \leq n_x} (|\xi_{i}| - \delta)_+^2 \geq \mathbb{E}_{X^n}\mathbb{E}_{Y^n \mid X^n}\max_{1 
   \leq i \leq n_x} (|\xi_{i}| - \delta)_+^21_{\mathcal{E}_2}.\\
   & \geq \mathbb{E}_{X^n}\mathbb{E}_{Y^n \mid X^n}\max_{1 \leq i \leq c_2nr^d} (|\xi_{i}| - 
   \delta)_+^21_{\mathcal{E}_2} = \mathbb{P}(\mathcal{E}_2)\mathbb{E}\max_{1 \leq i \leq c_2nr^d} (|\xi_{i}| - 
   \delta)_+^2\\
   & \gtrsim \mathbb{E}\max_{1 \leq i \leq c_2nr^d} (|\xi_{i}| - \delta)_+^2.
\end{split}
\end{equation*}
Using Markov's inequality, we obtain the following lower bound for any $t>0$:
\begin{equation}\label{eq:nianqing7}
\begin{split}
   \mathbb{E}\max_{1 \leq i \leq c_2nr^d} (|\xi_i| - \delta)_+^2 & \geq t \times \mathbb{P}\left\{ \max_{1 \leq 
   i \leq c_2nr^d} (|\xi_i| - \delta)_+^2 \geq t \right\} \\
     & = t \times \left(  1- \left(\mathbb{P}\left\{ (|\xi_i| - \delta)_+^2 \leq t \right\} \right)^{c_2nr^d} 
     \right)\\
     & = t \times \left\{  1- \left[1 - 2 \tilde{\Phi}\left(\frac{\sqrt{t}+ \delta}{\sigma}\right) \right]^{c_2nr^d} \right\},
\end{split}
\end{equation}
where the last step follows from
\begin{equation*}
  \mathbb{P}\left[ (|\xi_i| - \delta)_+^2 \leq t \right] = \mathbb{P}\left[ (|\xi_i| - \delta)_+ \leq \sqrt{t} 
  \right] = \mathbb{P}\left[ |\xi_i|  \leq \sqrt{t} + \delta\right] =  1 - 2 \tilde{\Phi}\left(\frac{\sqrt{t}+ 
  \delta}{\sigma}\right).
\end{equation*}

Since $nr^d \to \infty$ and $\delta$ is upper bounded by a constant, we see that the condition $\tilde{\Phi}( 
\delta/\sigma) \geq 1/(2c_2nr^d)$ must hold when $n$ is large enough. Suppose $t$ satisfies the 
condition
\begin{equation}\label{eq:nianqing8}
  \tilde{\Phi}\left(\frac{\sqrt{t} + \delta}{\sigma}\right) \geq \frac{1}{2c_2nr^d}.
\end{equation}
Then, from \eqref{eq:nianqing7}, we obtain
\begin{equation*}
  \begin{split}
       & \mathbb{E}\max_{1 \leq i \leq c_2nr^d} (|\xi_i| - \delta)_+^2  \geq t\left[ 1 - \left( 1 - 
       \frac{1}{c_2nr^d} \right)^{c_2nr^d} \right] \sim t \left[1 - \exp(-1) \right] \gtrsim t.
  \end{split}
\end{equation*}
Thus, to complete the proof, it remains to choose $t$ as large as possible while ensuring that condition 
\eqref{eq:nianqing8} holds. Based on Mills ratio, we have
\begin{equation*}
  \begin{split}
     \tilde{\Phi}(\lambda) & \geq \frac{1}{2\lambda}\phi(\lambda)  = \frac{\exp[(c- 
     \frac{1}{2})\lambda^2]}{2\sqrt{2\pi}\lambda} \exp(-c\lambda^2)  \gtrsim \exp(-c\lambda^2)
  \end{split}
\end{equation*}
for any $c> 1/2$ and $\lambda \to \infty$. Thus, we set $t = [( \sigma\sqrt{\frac{1}{c}\log (2c_2nr^d)} - \delta )_+]^2$. Since $\delta$ is upper bounded by a constant and $nr^d \to \infty$, we obtain the lower bound
$$
\mathbb{E}\max_{1 \leq i \leq c_2nr^d} (|\xi_i| - \delta)_+^2 \gtrsim t \gtrsim \log (nr^d),
$$
which completes the proof.

\subsubsection{Upper bound}

Using a similar technique to that in the proof of Lemma~\ref{lem:basic1}, when the constant $C_4$ is chosen sufficiently large, there exist two constants $c,\alpha>0$ such that 
\begin{equation}\label{eq:E3}
  \mathbb{P}\left\{  \forall x \in \mathcal{X}: n_x \leq c nr^d \right\} \geq 1 - n^{-\alpha}. 
\end{equation}
We denote the event in (\ref{eq:E3}) by $\mathcal{E}_3$. Then, for any $x \in \mathcal{X}$, we have 
\begin{subequations}
\begin{align}
   \mathbb{E}\max_{i \in \mathcal{S}_p(x,r)} (|\xi_{i}| - \delta)_+^2 &  =\mathbb{E}_{X^n}\mathbb{E}_{Y^n 
   \mid X^n}\max_{1 \leq i \leq n_x} (|\xi_{i}| - \delta)_+^2 \nonumber \\
     & =  \mathbb{E}_{X^n}\mathbb{E}_{Y^n \mid X^n}\max_{1 \leq i \leq n_x} (|\xi_{i}| - \delta)_+^2 1_{\{ 
   \mathcal{E}_3 \}} \label{eqn:line-3} \\
     & + \mathbb{E}_{X^n}\mathbb{E}_{Y^n \mid X^n}\max_{1 \leq i \leq n_x} (|\xi_{i}| - 
   \delta)_+^2 1_{\{ \mathcal{E}_3^c \}}. \label{eqn:line-4}
\end{align}
\end{subequations}
The term (\ref{eqn:line-3}) can be upper bounded as  
\begin{equation*}
  \begin{split}
     \mathbb{E}_{X^n}\mathbb{E}_{Y^n \mid X^n}\max_{1 \leq i \leq n_x} (|\xi_{i}| - \delta)_+^2 1_{\{ 
   \mathcal{E}_3 \}} & \leq \mathbb{E}\max_{1 \leq i \leq c nr^d} (|\xi_{i}| - \delta)_+^2  \lesssim \log (nr^d), 
  \end{split}
\end{equation*}
where the last inequality follows from the same reasoning as in (\ref{eq:bushuo1}). The term (\ref{eqn:line-4}) has the upper bound 
\begin{equation*}
  \mathbb{E}_{X^n}\mathbb{E}_{Y^n \mid X^n}\max_{1 \leq i \leq n_x} (|\xi_{i}| - 
   \delta)_+^2 1_{\{ \mathcal{E}_3^c \}} \leq \mathbb{P}(\mathcal{E}_3^c)\mathbb{E}\max_{1 \leq i \leq n}|\xi_{i}|^2 \lesssim (\log n) n^{-\alpha} \lesssim \log (nr^d),  
\end{equation*}
where the last inequality follows from the assumption $nr^d \geq C_4\log n $. This completes the proof of the upper bound. 

\section{Proof of Example~\ref{exp:shrink}}\label{sec:proof_exmaple}

We first establish the inequality 
\begin{equation}\label{eq:yuandana1}
  (r+\tau_n)^{2(1\wedge \beta)} + n^{-\frac{2\beta}{2\beta+d}} \lesssim r^{2(1\wedge \beta)} + 
  n^{-\frac{2\beta}{2\beta+d}} 
\end{equation}
under the condition $0 < \tau_n \lesssim n^{-\frac{\beta}{(2\beta+d)(1 \wedge \beta)} \vee 
\frac{4\beta+d}{d(2\beta+d)}}$. If $r \gtrsim \tau_n$, then $(r+\tau_n)^{2(1\wedge \beta)} \asymp r^{2(1\wedge 
\beta)}$, and the inequality follows immediately. It remains to consider the case $r \lesssim \tau_n$. We have 
\begin{equation*}
\begin{split}
   (r+\tau_n)^{2(1\wedge \beta)} + n^{-\frac{2\beta}{2\beta+d}} & \lesssim \tau_n^{2(1\wedge \beta)} + 
   n^{-\frac{2\beta}{2\beta+d}} \lesssim  n^{-\frac{2\beta}{2\beta+d}}  \\
     & \lesssim r^{2(1\wedge \beta)} + n^{-\frac{2\beta}{2\beta+d}}, 
\end{split}
\end{equation*}
where the second inequality uses $\tau_n \lesssim n^{-\frac{\beta}{(2\beta+d)(1 \wedge \beta)} }$ 

Next, we establish the inequality  
\begin{equation}\label{eq:yuandana2}
  \int_{\mathcal{X}}\mathbb{E}\max_{i \in \mathcal{S}_p(x,r + \tau_n)} (|\xi_{i}| - \delta)_+^2dx \lesssim 
  n^{-\frac{2\beta}{2\beta+d}} + \int_{\mathcal{X}}\mathbb{E}\max_{i \in \mathcal{S}_p(x,r)} (|\xi_{i}| - 
  \delta)_+^2dx. 
\end{equation}
When $r \gtrsim \tau_n$, we have 
\begin{equation*}
  \int_{\mathcal{X}}\mathbb{E}\max_{i \in \mathcal{S}_p(x,r + \tau_n)} (|\xi_{i}| - \delta)_+^2dx \lesssim 
  \int_{\mathcal{X}}\mathbb{E}\max_{i \in \mathcal{S}_p(x,2r)} (|\xi_{i}| - \delta)_+^2 dx \asymp 
  \int_{\mathcal{X}}\mathbb{E}\max_{i \in \mathcal{S}_p(x,r)} (|\xi_{i}| - \delta)_+^2dx, 
\end{equation*}
since enlarging the adversarial radius by a constant factor does not change the asymptotic order of the quantity. 
Consequently, it suffices to consider the case $r \lesssim \tau_n$. In this case, we have 
\begin{equation*}
  \int_{\mathcal{X}}\mathbb{E}\max_{i \in \mathcal{S}_p(x,r + \tau_n)} (|\xi_{i}| - \delta)_+^2dx \lesssim 
  \int_{\mathcal{X}}\mathbb{E}\max_{i \in \mathcal{S}_p(x,2\tau_n)} (|\xi_{i}| - \delta)_+^2 dx \lesssim 
  \int_{\mathcal{X}}\mathbb{E}\max_{i \in \mathcal{S}_p(x,2\tau_n)} \xi_{i}^2 dx. 
\end{equation*}
Note that $n\tau_n^d \lesssim n \times n^{- \frac{4\beta+d}{2\beta+d}} = n^{-\frac{2\beta}{2\beta+d}} \to 0$. 
Applying the results of Section~\ref{sec:proof_yuandan}, we conclude that  
\begin{equation*}
  \int_{\mathcal{X}}\mathbb{E}\max_{i \in \mathcal{S}_p(x,2\tau_n)} \xi_{i}^2 dx \lesssim n\tau_n^d \lesssim 
  n^{-\frac{2\beta}{2\beta+d}} . 
\end{equation*}
This proves \eqref{eq:yuandana2}.  

Finally, combining \eqref{eq:yuandana1} and \eqref{eq:yuandana2}, we conclude that the $\delta$-interpolating 
estimator $\hat{f}_{\delta,\tau_n}$ attains the minimax lower bound in \eqref{eq:lower_bound}.

\newpage

\bibliographystyle{apalike}
\bibliography{ITbibfile}

@inproceedings{Belkin2019,
  author       = {Mikhail Belkin and
                  Alexander Rakhlin and
                  Alexandre B. Tsybakov},
  title        = {Does data interpolation contradict statistical optimality?},
  booktitle    = {The 22nd International Conference on Artificial Intelligence and Statistics},
  pages        = {1611--1619},
  publisher    = {{PMLR}},
  year         = {2019},
  url          = {http://proceedings.mlr.press/v89/belkin19a.html},
  bibsource    = {dblp computer science bibliography, https://dblp.org}
}

@article{ibragimov1982bounds,
  title={Bounds for the risks of non-parametric regression estimates},
  author={Ibragimov, IA and Khas’ minskii, RZ},
  journal={Theory of Probability $\&$ Its Applications},
  volume={27},
  number={1},
  pages={84--99},
  year={1982},
  publisher={SIAM}
}

@article{Yang1999Information,
author = {Yuhong Yang and Andrew Barron},
title = {Information-theoretic determination of minimax rates of convergence},
volume = {27},
journal = {The Annals of Statistics},
number = {5},
publisher = {Institute of Mathematical Statistics},
pages = {1564--1599},
year = {1999},
doi = {10.1214/aos/1017939142},
URL = {https://doi.org/10.1214/aos/1017939142}
}

@article{birge1986estimating,
  title={On estimating a density using Hellinger distance and some other strange facts},
  author={Birg{\'e}, Lucien},
  journal={Probability theory and related fields},
  volume={71},
  number={2},
  pages={271--291},
  year={1986},
  publisher={Springer}
}

@misc{Quinlan1993,
  author       = {Quinlan, R.},
  title        = {{Auto MPG}},
  year         = {1993},
  howpublished = {UCI Machine Learning Repository},
  note         = {{DOI}: https://doi.org/10.24432/C5859H}
}

@inproceedings{Mallinar2022,
 author = {Mallinar, Neil and Simon, James and Abedsoltan, Amirhesam and Pandit, Parthe and Belkin, Misha and Nakkiran, Preetum},
 booktitle = {Advances in Neural Information Processing Systems},
 pages = {1182--1195},
 publisher = {Curran Associates, Inc.},
 title = {Benign, Tempered, or Catastrophic: Toward a Refined Taxonomy of Overfitting},
 url = {https://proceedings.neurips.cc/paper_files/paper/2022/file/08342dc6ab69f23167b4123086ad4d38-Paper-Conference.pdf},
 volume = {35},
 year = {2022}
}

@article{kohler2025rate,
      title={On the rate of convergence of an over-parametrized deep neural network regression estimate learned by gradient descent},
      author={Michael Kohler},
      year={2025},
      journal={arXiv preprint arXiv:2504.03405},
}

@article{Mucke2025,
   author = {M\"{u}cke, Nicole and Steinwart, Ingo},
   title = {Empirical risk minimization in the interpolating regime with application to neural network learning},
   journal = {Machine Learning},
   volume = {114},
   number = {4},
   pages = {102},
   ISSN = {1573-0565},
   DOI = {10.1007/s10994-025-06738-9},
   url = {https://doi.org/10.1007/s10994-025-06738-9},
   year = {2025},
   type = {Journal Article}
}

@inproceedings{Haas2023,
 author = {Haas, Moritz and Holzm\"{u}ller, David and Luxburg, Ulrike and Steinwart, Ingo},
 booktitle = {Advances in Neural Information Processing Systems},
 pages = {20763--20826},
 title = {Mind the spikes: Benign overfitting of kernels and neural networks in fixed dimension},
 url = {https://proceedings.neurips.cc/paper_files/paper/2023/file/421f83663c02cdaec8c3c38337709989-Paper-Conference.pdf},
 volume = {36},
 year = {2023}
}

@InProceedings{Wang2019Language,
  title = 	 {Improving Neural Language Modeling via Adversarial Training},
  author =       {Wang, Dilin and Gong, Chengyue and Liu, Qiang},
  booktitle = 	 {Proceedings of the 36th International Conference on Machine Learning},
  pages = 	 {6555--6565},
  year = 	 {2019},
  volume = 	 {97},
  series = 	 {Proceedings of Machine Learning Research},
  month = 	 {09--15 Jun},
  publisher =    {PMLR},
  url = 	 {https://proceedings.mlr.press/v97/wang19f.html}
}

@article{kohler2021rate,
  title={On the rate of convergence of fully connected deep neural network regression estimates},
  author={Kohler, Michael and Langer, Sophie},
  journal={The Annals of Statistics},
  volume={49},
  number={4},
  pages={2231--2249},
  year={2021},
  publisher={JSTOR}
}

@book{Le_Cambook,
title = {Asymptotic Methods in Statistical Decision Theory},
author = {Le Cam, Lucien},
year = {1986},
publisher = {Springer New York, NY}
}

@article{DEVROYE1998209,
title = {The {H}ilbert kernel regression estimate},
journal = {Journal of Multivariate Analysis},
volume = {65},
number = {2},
pages = {209--227},
year = {1998},
issn = {0047-259X},
doi = {https://doi.org/10.1006/jmva.1997.1725},
url = {https://www.sciencedirect.com/science/article/pii/S0047259X97917255},
author = {Luc Devroye and Laszlo Györfi and Adam Krzyżak}
}

@inproceedings{Belkin2018Overfitting,
  author       = {Mikhail Belkin and
                  Daniel J. Hsu and
                  Partha Mitra},
  title        = {Overfitting or perfect fitting? {R}isk bounds for classification and
                  regression rules that interpolate},
  booktitle    = {Advances in Neural Information Processing Systems 31},
  pages        = {2306--2317},
  year         = {2018},
  url          = {https://proceedings.neurips.cc/paper/2018/hash/e22312179bf43e61576081a2f250f845-Abstract.html},
  bibsource    = {dblp computer science bibliography, https://dblp.org}
}

@article{Chhor2024Benign,
   author = {Chhor, Julien and Sigalla, Suzanne and Tsybakov, Alexandre B.},
   title = {Benign overfitting and adaptive nonparametric regression},
   journal = {Probability Theory and Related Fields},
   volume = {189},
   number = {3},
   pages = {949--980},
   ISSN = {1432-2064},
   DOI = {10.1007/s00440-024-01278-0},
   url = {https://doi.org/10.1007/s00440-024-01278-0},
   year = {2024},
   type = {Journal Article}
}

@article{Peng2025aos,
  title={Adversarial learning for nonparametric regression: {M}inimax rate and adaptive estimation},
  author={Peng, Jingfu and Yang, Yuhong},
  journal={arXiv preprint arXiv:2506.01267},
  year={2025}
}

@inproceedings{Sanyal2021,
  author       = {Amartya Sanyal and
                  Puneet K. Dokania and
                  Varun Kanade and
                  Philip H. S. Torr},
  title        = {How benign is benign overfitting?},
  booktitle    = {9th International Conference on Learning Representations},
  year         = {2021},
  url          = {https://openreview.net/forum?id=g-wu9TMPODo},
  bibsource    = {dblp computer science bibliography, https://dblp.org}
}

@article{Srivastava2014,
  author  = {Nitish Srivastava and Geoffrey Hinton and Alex Krizhevsky and Ilya Sutskever and Ruslan Salakhutdinov},
  title   = {Dropout: {A} simple way to prevent neural networks from overfitting},
  journal = {Journal of Machine Learning Research},
  year    = {2014},
  volume  = {15},
  number  = {56},
  pages   = {1929--1958},
  url     = {http://jmlr.org/papers/v15/srivastava14a.html}
}

@article{Chaudhari_2019,
doi = {10.1088/1742-5468/ab39d9},
url = {https://dx.doi.org/10.1088/1742-5468/ab39d9},
year = {2019},
month = {dec},
publisher = {IOP Publishing and SISSA},
volume = {2019},
number = {12},
pages = {124018},
author = {Chaudhari, Pratik and Choromanska, Anna and Soatto, Stefano and LeCun, Yann and Baldassi, Carlo and Borgs, Christian and Chayes, Jennifer and Sagun, Levent and Zecchina, Riccardo},
title = {Entropy-SGD: biasing gradient descent into wide valleys*},
journal = {Journal of Statistical Mechanics: Theory and Experiment}
}

@InProceedings{Ma2018,
  title = 	 {The Power of Interpolation: Understanding the Effectiveness of {SGD} in Modern Over-parametrized Learning},
  author =       {Ma, Siyuan and Bassily, Raef and Belkin, Mikhail},
  booktitle = 	 {Proceedings of the 35th International Conference on Machine Learning},
  pages = 	 {3325--3334},
  year = 	 {2018},
  volume = 	 {80},
  series = 	 {Proceedings of Machine Learning Research},
  month = 	 {10--15 Jul},
  publisher =    {PMLR},
  url = 	 {https://proceedings.mlr.press/v80/ma18a.html}
}

@article{neyshabur2014search,
  title={In search of the real inductive bias: {On} the role of implicit regularization in deep learning},
  author={Neyshabur, Behnam and Tomioka, Ryota and Srebro, Nathan},
  journal={arXiv preprint arXiv:1412.6614},
  year={2014}
}

@inproceedings{savarese2019infinite,
  title={How do infinite width bounded norm networks look in function space?},
  author={Savarese, Pedro and Evron, Itay and Soudry, Daniel and Srebro, Nathan},
  booktitle={Conference on Learning Theory},
  pages={2667--2690},
  year={2019},
  organization={PMLR}
}

@article{Ergen2021,
  author  = {Tolga Ergen and Mert Pilanci},
  title   = {Convex Geometry and Duality of Over-parameterized Neural Networks},
  journal = {Journal of Machine Learning Research},
  year    = {2021},
  volume  = {22},
  number  = {212},
  pages   = {1--63},
  url     = {http://jmlr.org/papers/v22/20-1447.html}
}

@article{hanin2021ridgeless,
  title={Ridgeless Interpolation with Shallow ReLU Networks in $1 D $ is Nearest Neighbor Curvature Extrapolation and Provably Generalizes on Lipschitz Functions},
  author={Hanin, Boris},
  journal={arXiv preprint arXiv:2109.12960},
  year={2021}
}

@inproceedings{boursier2023penalising,
title={Penalising the biases in norm regularisation enforces sparsity},
author={Etienne Boursier and Nicolas Flammarion},
booktitle={Thirty-seventh Conference on Neural Information Processing Systems},
year={2023},
url={https://openreview.net/forum?id=JtIqG47DAQ}
}

@inproceedings{
joshi2024noisy,
title={Noisy Interpolation Learning with Shallow Univariate Re{LU} Networks},
author={Nirmit Joshi and Gal Vardi and Nathan Srebro},
booktitle={The Twelfth International Conference on Learning Representations},
year={2024},
url={https://openreview.net/forum?id=GTUoTJXPBf}
}

@inproceedings{Ji2021,
  author       = {Ziwei Ji and
                  Justin D. Li and
                  Matus Telgarsky},
  editor       = {Marc'Aurelio Ranzato and
                  Alina Beygelzimer and
                  Yann N. Dauphin and
                  Percy Liang and
                  Jennifer Wortman Vaughan},
  title        = {Early-stopped neural networks are consistent},
  booktitle    = {Advances in Neural Information Processing Systems 34: Annual Conference
                  on Neural Information Processing Systems 2021, NeurIPS 2021, December
                  6-14, 2021, virtual},
  pages        = {1805--1817},
  year         = {2021},
  url          = {https://proceedings.neurips.cc/paper/2021/hash/0e1ebad68af7f0ae4830b7ac92bc3c6f-Abstract.html},
  bibsource    = {dblp computer science bibliography, https://dblp.org}
}

@InProceedings{Gupta2015,
  title = 	 {Deep learning with limited numerical precision},
  author = 	 {Gupta, Suyog and Agrawal, Ankur and Gopalakrishnan, Kailash and Narayanan, Pritish},
  booktitle = 	 {Proceedings of the 32nd International Conference on Machine Learning},
  pages = 	 {1737--1746},
  year = 	 {2015},
  editor = 	 {Bach, Francis and Blei, David},
  volume = 	 {37},
  series = 	 {Proceedings of Machine Learning Research},
  address = 	 {Lille, France},
  month = 	 {07--09 Jul},
  publisher =    {PMLR},
  url = 	 {https://proceedings.mlr.press/v37/gupta15.html}
}

@inproceedings{Du2019,
  author       = {Simon S. Du and
                  Jason D. Lee and
                  Haochuan Li and
                  Liwei Wang and
                  Xiyu Zhai},
  editor       = {Kamalika Chaudhuri and
                  Ruslan Salakhutdinov},
  title        = {Gradient descent finds global minima of deep neural networks},
  booktitle    = {Proceedings of the 36th International Conference on Machine Learning},
  series       = {Proceedings of Machine Learning Research},
  volume       = {97},
  pages        = {1675--1685},
  publisher    = {{PMLR}},
  year         = {2019},
  url          = {http://proceedings.mlr.press/v97/du19c.html},
  bibsource    = {dblp computer science bibliography, https://dblp.org}
}

@INPROCEEDINGS{Kawaguchi2019,
  author={Kawaguchi, Kenji and Huang, Jiaoyang},
  booktitle={57th Annual Allerton Conference on Communication, Control, and Computing (Allerton)},
  title={Gradient Descent Finds Global Minima for Generalizable Deep Neural Networks of Practical Sizes},
  year={2019},
  volume={},
  number={},
  pages={92-99},
  doi={10.1109/ALLERTON.2019.8919696}}

@InProceedings{Allen-Zhu2019,
  title = 	 {A convergence theory for deep learning via over-parameterization},
  author =       {Allen-Zhu, Zeyuan and Li, Yuanzhi and Song, Zhao},
  booktitle = 	 {Proceedings of the 36th International Conference on Machine Learning},
  pages = 	 {242--252},
  year = 	 {2019},
  volume = 	 {97},
  series = 	 {Proceedings of Machine Learning Research},
  month = 	 {09--15 Jun},
  publisher =    {PMLR},
  url = 	 {https://proceedings.mlr.press/v97/allen-zhu19a.html}
}

@InProceedings{Zhang2019,
  title = 	 {Theoretically principled trade-off between robustness and accuracy},
  author =       {Zhang, Hongyang and Yu, Yaodong and Jiao, Jiantao and Xing, Eric and Ghaoui, Laurent El and Jordan, Michael},
  booktitle = 	 {Proceedings of the 36th International Conference on Machine Learning},
  pages = 	 {7472--7482},
  year = 	 {2019},
  url = 	 {https://proceedings.mlr.press/v97/zhang19p.html}
}

@inproceedings{Donhauser2021,
 author = {Donhauser, Konstantin and Tifrea, Alexandru and Aerni, Michael and Heckel, Reinhard and Yang, Fanny},
 booktitle = {Advances in Neural Information Processing Systems},
 pages = {23465--23477},
 title = {Interpolation can hurt robust generalization even when there is no noise},
 url = {https://proceedings.neurips.cc/paper_files/paper/2021/file/c4f2c88e16a579900657c18726641c81-Paper.pdf},
 volume = {34},
 year = {2021}
}

@book{Tsybakov2009,
author="Tsybakov, Alexandre B.",
Title="Introduction to Nonparametric Estimation",
year="2009",
publisher="Springer New York"
}

@book{Gyorfi2002,
Title="A Distribution-Free Theory of Nonparametric Regression",
year="2002",
publisher="Springer New York",
author={L\'{a}szl\'{o} Gy\"{o}rfi and Michael Kohler and Adam Krzy\.{z}ak and Harro Walk}
}

@article{Schmidt-Hieber2020,
author = {Johannes Schmidt-Hieber},
title = {Nonparametric regression using deep neural networks with {ReLU} activation function},
volume = {48},
journal = {The Annals of Statistics},
number = {4},
publisher = {Institute of Mathematical Statistics},
pages = {1875--1897},
year = {2020},
doi = {10.1214/19-AOS1875},
URL = {https://doi.org/10.1214/19-AOS1875}
}

@inproceedings{Goodfellow2015,title	= {Explaining and harnessing adversarial examples},author	= {Ian Goodfellow and Jonathon Shlens and Christian Szegedy},year	= {2015},URL	= {http://arxiv.org/abs/1412.6572},booktitle	= {International Conference on Learning Representations}}

@article{Zhang2016understanding,
      title={Understanding deep learning requires rethinking generalization},
      author={Chiyuan Zhang and Samy Bengio and Moritz Hardt and Benjamin Recht and Oriol Vinyals},
      year={2016},
      journal={arXiv preprint arXiv:1611.03530},
}

@article{Belkin2019Reconciling,
author = {Mikhail Belkin and Daniel Hsu  and Siyuan Ma  and Soumik Mandal },
title = {Reconciling modern machine-learning practice and the classical bias-variance trade-off},
journal = {Proceedings of the National Academy of Sciences},
volume = {116},
number = {32},
pages = {15849--15854},
year = {2019},
doi = {10.1073/pnas.1903070116},
URL = {https://www.pnas.org/doi/abs/10.1073/pnas.1903070116},
eprint = {https://www.pnas.org/doi/pdf/10.1073/pnas.1903070116}}

@InProceedings{Belkin2018To,
  title = 	 {To understand deep learning we need to understand kernel learning},
  author =       {Belkin, Mikhail and Ma, Siyuan and Mandal, Soumik},
  booktitle = 	 {Proceedings of the 35th International Conference on Machine Learning},
  pages = 	 {541--549},
  year = 	 {2018},
  volume = 	 {80},
  series = 	 {Proceedings of Machine Learning Research},
  month = 	 {10--15 Jul},
  publisher =    {PMLR},
  url = 	 {https://proceedings.mlr.press/v80/belkin18a.html}
}

@article{Liang2020Just,
author = {Tengyuan Liang and Alexander Rakhlin},
title = {Just interpolate: Kernel {“Ridgeless”} regression can generalize},
volume = {48},
journal = {The Annals of Statistics},
number = {3},
publisher = {Institute of Mathematical Statistics},
pages = {1329 -- 1347},
year = {2020},
doi = {10.1214/19-AOS1849},
URL = {https://doi.org/10.1214/19-AOS1849}
}

@article{CHEN2015447,
title = {Optimal uniform convergence rates and asymptotic normality for series estimators under weak dependence and weak conditions},
journal = {Journal of Econometrics},
volume = {188},
number = {2},
pages = {447--465},
year = {2015},
note = {Heterogeneity in Panel Data and in Nonparametric Analysis in honor of Professor Cheng Hsiao},
issn = {0304-4076},
doi = {https://doi.org/10.1016/j.jeconom.2015.03.010},
url = {https://www.sciencedirect.com/science/article/pii/S0304407615000792},
author = {Xiaohong Chen and Timothy M. Christensen}
}

@InProceedings{Arnould2023,
  title = 	 {Is interpolation benign for random forest regression?},
  author =       {Arnould, Ludovic and Boyer, Claire and Scornet, Erwan},
  booktitle = 	 {Proceedings of The 26th International Conference on Artificial Intelligence and Statistics},
  pages = 	 {5493--5548},
  year = 	 {2023},
  volume = 	 {206},
  series = 	 {Proceedings of Machine Learning Research},
  month = 	 {25--27 Apr},
  publisher =    {PMLR},
  url = 	 {https://proceedings.mlr.press/v206/arnould23a.html}
}

@article{BELLONI2015345,
title = {Some new asymptotic theory for least squares series: Pointwise and uniform results},
journal = {Journal of Econometrics},
volume = {186},
number = {2},
pages = {345--366},
year = {2015},
note = {High Dimensional Problems in Econometrics},
issn = {0304-4076},
doi = {https://doi.org/10.1016/j.jeconom.2015.02.014},
url = {https://www.sciencedirect.com/science/article/pii/S030440761500038X},
author = {Alexandre Belloni and Victor Chernozhukov and Denis Chetverikov and Kengo Kato}
}

@InProceedings{Rakhlin2019Consistency,
  title = 	 {Consistency of Interpolation with Laplace Kernels is a High-Dimensional Phenomenon},
  author =       {Rakhlin, Alexander and Zhai, Xiyu},
  booktitle = 	 {Proceedings of the Thirty-Second Conference on Learning Theory},
  pages = 	 {2595--2623},
  year = 	 {2019},
  volume = 	 {99},
  series = 	 {Proceedings of Machine Learning Research},
  month = 	 {25--28 Jun},
  publisher =    {PMLR},
  url = 	 {https://proceedings.mlr.press/v99/rakhlin19a.html}
}

@inproceedings{Szegedy2014,
  author       = {Christian Szegedy and
                  Wojciech Zaremba and
                  Ilya Sutskever and
                  Joan Bruna and
                  Dumitru Erhan and
                  Ian J. Goodfellow and
                  Rob Fergus},
  title        = {Intriguing properties of neural networks},
  booktitle    = {2nd International Conference on Learning Representations},
  year         = {2014},
  url          = {http://arxiv.org/abs/1312.6199},
  bibsource    = {dblp computer science bibliography, https://dblp.org}
}

@ARTICLE{Muthukumar2020,
  author={Muthukumar, Vidya and Vodrahalli, Kailas and Subramanian, Vignesh and Sahai, Anant},
  journal={IEEE Journal on Selected Areas in Information Theory},
  title={Harmless Interpolation of Noisy Data in Regression},
  year={2020},
  volume={1},
  number={1},
  pages={67-83},
  doi={10.1109/JSAIT.2020.2984716}}

@article{Belkin2020Two,
author = {Belkin, Mikhail and Hsu, Daniel and Xu, Ji},
title = {Two Models of Double Descent for Weak Features},
journal = {SIAM Journal on Mathematics of Data Science},
volume = {2},
number = {4},
pages = {1167-1180},
year = {2020},
doi = {10.1137/20M1336072},

URL = {

        https://doi.org/10.1137/20M1336072



},
eprint = {

        https://doi.org/10.1137/20M1336072



}
}

@article{Hastie2022Surprises,
author = {Trevor Hastie and Andrea Montanari and Saharon Rosset and Ryan J. Tibshirani},
title = {{Surprises in high-dimensional ridgeless least squares interpolation}},
volume = {50},
journal = {The Annals of Statistics},
number = {2},
publisher = {Institute of Mathematical Statistics},
pages = {949 -- 986},
year = {2022},
doi = {10.1214/21-AOS2133},
URL = {https://doi.org/10.1214/21-AOS2133}
}

@article{
Bartlett2020Benign,
author = {Peter L. Bartlett  and Philip M. Long  and Gábor Lugosi  and Alexander Tsigler },
title = {Benign overfitting in linear regression},
journal = {Proceedings of the National Academy of Sciences},
volume = {117},
number = {48},
pages = {30063-30070},
year = {2020},
doi = {10.1073/pnas.1907378117},
URL = {https://www.pnas.org/doi/abs/10.1073/pnas.1907378117},
eprint = {https://www.pnas.org/doi/pdf/10.1073/pnas.1907378117}}

@article{Lecue2024geometrical,
   author = {Lecu\'{e}, Guillaume and Shang, Zong},
   title = {A geometrical viewpoint on the benign overfitting property of the minimum $\ell _2$-norm interpolant estimator and its universality},
   journal = {Probability Theory and Related Fields},
   ISSN = {1432-2064},
   DOI = {10.1007/s00440-024-01336-7},
   url = {https://doi.org/10.1007/s00440-024-01336-7},
   year = {2024},
   type = {Journal Article}
}

@article{Chinot2020robustness,
author = {Geoffrey Chinot and Matthieu Lerasle},
title = {{On the robustness of the minimim  interpolator}},
volume = {31},
journal = {Bernoulli},
number = {3},
publisher = {Bernoulli Society for Mathematical Statistics and Probability},
pages = {1693--1708},
year = {2025},
doi = {10.3150/22-BEJ1473},
URL = {https://doi.org/10.3150/22-BEJ1473}
}

@article{Mei2022,
author = {Mei, Song and Montanari, Andrea},
title = {The Generalization Error of Random Features Regression: Precise Asymptotics and the Double Descent Curve},
journal = {Communications on Pure and Applied Mathematics},
volume = {75},
number = {4},
pages = {667-766},
doi = {https://doi.org/10.1002/cpa.22008},
url = {https://onlinelibrary.wiley.com/doi/abs/10.1002/cpa.22008},
eprint = {https://onlinelibrary.wiley.com/doi/pdf/10.1002/cpa.22008},
year = {2022}
}

@article{Tsigler2023,
  author  = {Alexander Tsigler and Peter L. Bartlett},
  title   = {Benign overfitting in ridge regression},
  journal = {Journal of Machine Learning Research},
  year    = {2023},
  volume  = {24},
  number  = {123},
  pages   = {1--76},
  url     = {http://jmlr.org/papers/v24/22-1398.html}
}

@ARTICLE{Su2019,
  author={Su, Jiawei and Vargas, Danilo Vasconcellos and Sakurai, Kouichi},
  journal={IEEE Transactions on Evolutionary Computation},
  title={One Pixel Attack for Fooling Deep Neural Networks},
  year={2019},
  volume={23},
  number={5},
  pages={828-841},
  doi={10.1109/TEVC.2019.2890858}}

@InProceedings{Biggio2013,
author="Biggio, Battista
and Corona, Igino
and Maiorca, Davide
and Nelson, Blaine
and {\v{S}}rndi{\'{c}}, Nedim
and Laskov, Pavel
and Giacinto, Giorgio
and Roli, Fabio",
title="Evasion attacks against machine learning at test time",
booktitle="Machine Learning and Knowledge Discovery in Databases",
year="2013",
publisher="Springer Berlin Heidelberg",
pages="387--402",
isbn="978-3-642-40994-3"
}

@inproceedings{Madry2018Towards,
  author       = {Aleksander Madry and
                  Aleksandar Makelov and
                  Ludwig Schmidt and
                  Dimitris Tsipras and
                  Adrian Vladu},
  title        = {Towards deep learning models resistant to adversarial attacks},
  booktitle    = {6th International Conference on Learning Representations},
  year         = {2018},
  url          = {https://openreview.net/forum?id=rJzIBfZAb},
  bibsource    = {dblp computer science bibliography, https://dblp.org}
}

@inproceedings{Wong2018,
  author       = {Eric Wong and
                  J. Zico Kolter},
  title        = {Provable Defenses against Adversarial Examples via the Convex Outer
                  Adversarial Polytope},
  booktitle    = {Proceedings of the 35th International Conference on Machine Learning},
  series       = {Proceedings of Machine Learning Research},
  volume       = {80},
  pages        = {5283--5292},
  publisher    = {{PMLR}},
  year         = {2018},
  url          = {http://proceedings.mlr.press/v80/wong18a.html},
  bibsource    = {dblp computer science bibliography, https://dblp.org}
}

@article{Peng2024,
author = {Jingfu Peng and Yuhong Yang},
journal={arXiv preprint arXiv:2410.09402},
title={Minimax rates of convergence for nonparametric regression under adversarial attacks},
year = {2024}
}

@article{Kohler2021,
author = {Michael Kohler and Adam Krzyżak},
title = {Over-parametrized deep neural networks minimizing the empirical risk do not generalize well},
volume = {27},
journal = {Bernoulli},
number = {4},
publisher = {Bernoulli Society for Mathematical Statistics and Probability},
pages = {2564--2597},
year = {2021},
doi = {10.3150/21-BEJ1323},
URL = {https://doi.org/10.3150/21-BEJ1323}
}

@ARTICLE{Das2025direct,
  author={Das, Santanu and Batra, Jatin and Srivastava, Piyush},
  journal={IEEE Transactions on Information Theory},
  title={A direct proof of a unified law of robustness for Bregman divergence losses},
  year={2025},
  volume={},
  number={},
  pages={1-1},
  doi={10.1109/TIT.2025.3567076}}

@InProceedings{Wu2023law,
  title = 	 {A law of robustness beyond isoperimetry},
  author =       {Wu, Yihan and Huang, Heng and Zhang, Hongyang},
  booktitle = 	 {Proceedings of the 40th International Conference on Machine Learning},
  pages = 	 {37439--37455},
  year = 	 {2023},
  volume = 	 {202},
  series = 	 {Proceedings of Machine Learning Research},
  month = 	 {23--29 Jul},
  publisher =    {PMLR},
  url = 	 {https://proceedings.mlr.press/v202/wu23g.html}
}

@inproceedings{Raghunathan2018Certified,
  author       = {Aditi Raghunathan and
                  Jacob Steinhardt and
                  Percy Liang},
  title        = {Certified Defenses against Adversarial Examples},
  booktitle    = {6th International Conference on Learning Representations},
  year         = {2018},
  url          = {https://openreview.net/forum?id=Bys4ob-Rb},
  bibsource    = {dblp computer science bibliography, https://dblp.org}
}

@inproceedings{Cohen2019Certified,
  author       = {Jeremy Cohen and
                  Elan Rosenfeld and
                  J. Zico Kolter},
  title        = {Certified Adversarial Robustness via Randomized Smoothing},
  booktitle    = {Proceedings of the 36th International Conference on Machine Learning},
  series       = {Proceedings of Machine Learning Research},
  volume       = {97},
  pages        = {1310--1320},
  publisher    = {{PMLR}},
  year         = {2019},
  url          = {http://proceedings.mlr.press/v97/cohen19c.html},
  bibsource    = {dblp computer science bibliography, https://dblp.org}
}

@inproceedings{Bubeck2011,
  author       = {S{\'{e}}bastien Bubeck and
                  Yuanzhi Li and
                  Dheeraj M. Nagaraj},
  title        = {A law of robustness for two-layers neural networks},
  booktitle    = {Conference on Learning Theory},
  series       = {Proceedings of Machine Learning Research},
  volume       = {134},
  pages        = {804--820},
  publisher    = {{PMLR}},
  year         = {2021},
  url          = {http://proceedings.mlr.press/v134/bubeck21a.html},
  bibsource    = {dblp computer science bibliography, https://dblp.org}
}

@article{Bubeck2023,
author = {Bubeck, S\'{e}bastien and Sellke, Mark},
title = {A Universal Law of Robustness via Isoperimetry},
year = {2023},
issue_date = {April 2023},
publisher = {Association for Computing Machinery},
address = {New York, NY, USA},
volume = {70},
number = {2},
issn = {0004-5411},
url = {https://doi.org/10.1145/3578580},
doi = {10.1145/3578580},
journal = {Journal of the ACM},
month = mar,
articleno = {10},
numpages = {18}
}

@article{GAIFFAS2007782,
title = {Sharp estimation in sup norm with random design},
journal = {Statistics \& Probability Letters},
volume = {77},
number = {8},
pages = {782-794},
year = {2007},
issn = {0167-7152},
doi = {https://doi.org/10.1016/j.spl.2006.11.017},
url = {https://www.sciencedirect.com/science/article/pii/S016771520700003X},
author = {St\'{e}phane Ga\"{\i}ffas}
}

@article{BERTIN2004225,
title = {Minimax exact constant in sup-norm for nonparametric regression with random design},
journal = {Journal of Statistical Planning and Inference},
volume = {123},
number = {2},
pages = {225-242},
year = {2004},
issn = {0378-3758},
doi = {https://doi.org/10.1016/S0378-3758(03)00154-X},
url = {https://www.sciencedirect.com/science/article/pii/S037837580300154X},
author = {Karine Bertin}
}

@article{zou2023universal,
  title={Universal and transferable adversarial attacks on aligned language models},
  author={Zou, Andy and Wang, Zifan and Carlini, Nicholas and Nasr, Milad and Kolter, J Zico and Fredrikson, Matt},
  journal={arXiv preprint arXiv:2307.15043},
  year={2023}
}

@InProceedings{Attias2023Real-Valued,
  title = 	 {Adversarially Robust {PAC} Learnability of Real-Valued Functions},
  author =       {Attias, Idan and Hanneke, Steve},
  booktitle = 	 {Proceedings of the 40th International Conference on Machine Learning},
  pages = 	 {1172--1199},
  year = 	 {2023},
  volume = 	 {202},
  series = 	 {Proceedings of Machine Learning Research},
  month = 	 {23--29 Jul},
  publisher =    {PMLR},
  url = 	 {https://proceedings.mlr.press/v202/attias23a.html}
}

@article{Cheng2024ridge,
author = {Chen Cheng and Andrea Montanari},
title = {Dimension free ridge regression},
volume = {52},
journal = {The Annals of Statistics},
number = {6},
publisher = {Institute of Mathematical Statistics},
pages = {2879--2912},
year = {2024},
doi = {10.1214/24-AOS2449},
URL = {https://doi.org/10.1214/24-AOS2449}
}

@inproceedings{buchholz2022kernel,
  title={Kernel interpolation in Sobolev spaces is not consistent in low dimensions},
  author={Buchholz, Simon},
  booktitle={Conference on Learning Theory},
  pages={3410--3440},
  year={2022},
  organization={PMLR}
}

@incollection{yu1997assouad,
  title={Assouad, {Fano}, and {Le Cam}},
  author={Yu, Bin},
  booktitle={Festschrift for Lucien Le Cam: research papers in probability and statistics},
  pages={423--435},
  year={1997},
  publisher={Springer}
}

@article{Li2023interpolation,
    author = {Li, Yicheng and Zhang, Haobo and Lin, Qian},
    title = {Kernel interpolation generalizes poorly},
    journal = {Biometrika},
    volume = {111},
    number = {2},
    pages = {715--722},
    year = {2023},
    month = {08},
    issn = {1464-3510},
    doi = {10.1093/biomet/asad048},
    url = {https://doi.org/10.1093/biomet/asad048},
    eprint = {https://academic.oup.com/biomet/article-pdf/111/2/715/57467427/asad048.pdf},
}

@book{Wainwright_2019, 
place={Cambridge}, 
series={Cambridge Series in Statistical and Probabilistic Mathematics}, 
title={High-Dimensional Statistics: A Non-Asymptotic Viewpoint}, 
publisher={Cambridge University Press}, 
author={Wainwright, Martin J.}, 
year={2019}, 
collection={Cambridge Series in Statistical and Probabilistic Mathematics}
}

@INPROCEEDINGS{Eykholt2018,
  author={Eykholt, Kevin and Evtimov, Ivan and Fernandes, Earlence and Li, Bo and Rahmati, Amir and Xiao, Chaowei and Prakash, Atul and Kohno, Tadayoshi and Song, Dawn},
  booktitle={2018 IEEE/CVF Conference on Computer Vision and Pattern Recognition},
  title={Robust Physical-World Attacks on Deep Learning Visual Classification},
  year={2018},
  volume={},
  number={},
  pages={1625-1634},
  doi={10.1109/CVPR.2018.00175}}

@article{Donoho1994,
    author = {Donoho, David L and Johnstone, Iain M},
    title = {Ideal spatial adaptation by wavelet shrinkage},
    journal = {Biometrika},
    volume = {81},
    number = {3},
    pages = {425--455},
    year = {1994},
    month = {09},
    issn = {0006-3444},
    doi = {10.1093/biomet/81.3.425},
    url = {https://doi.org/10.1093/biomet/81.3.425},
    eprint = {https://academic.oup.com/biomet/article-pdf/81/3/425/26079146/81.3.425.pdf},
}

@InProceedings{Papangelou2019,
author="Papangelou, Konstantinos
and Sechidis, Konstantinos
and Weatherall, James
and Brown, Gavin",
title="Toward an Understanding of Adversarial Examples in Clinical Trials",
booktitle="Machine Learning and Knowledge Discovery in Databases",
year="2019",
publisher="Springer International Publishing",
address="Cham",
pages="35--51",
isbn="978-3-030-10925-7"
}

@InProceedings{Javanmard2020,
  title = 	 {Precise Tradeoffs in Adversarial Training for Linear Regression},
  author =       {Javanmard, Adel and Soltanolkotabi, Mahdi and Hassani, Hamed},
  booktitle = 	 {Proceedings of Thirty Third Conference on Learning Theory},
  pages = 	 {2034--2078},
  year = 	 {2020},
  volume = 	 {125},
  series = 	 {Proceedings of Machine Learning Research},
  month = 	 {09--12 Jul},
  publisher =    {PMLR},
  url = 	 {https://proceedings.mlr.press/v125/javanmard20a.html}
}

@article{Audibert2007,
author = {Jean-Yves Audibert and Alexandre B. Tsybakov},
title = {Fast learning rates for plug-in classifiers},
volume = {35},
journal = {The Annals of Statistics},
number = {2},
publisher = {Institute of Mathematical Statistics},
pages = {608--633},
year = {2007},
doi = {10.1214/009053606000001217},
URL = {https://doi.org/10.1214/009053606000001217}
}

@article{Stone1982,
author = {Charles J. Stone},
title = {Optimal global rates of convergence for nonparametric regression},
volume = {10},
journal = {The Annals of Statistics},
number = {4},
publisher = {Institute of Mathematical Statistics},
pages = {1040--1053},
year = {1982},
doi = {10.1214/aos/1176345969},
URL = {https://doi.org/10.1214/aos/1176345969}
}

@InProceedings{Xing2021,
  title = 	 {Adversarially robust estimate and risk analysis in linear regression},
  author =       {Xing, Yue and Zhang, Ruizhi and Cheng, Guang},
  booktitle = 	 {Proceedings of The 24th International Conference on Artificial Intelligence and Statistics},
  pages = 	 {514--522},
  year = 	 {2021},
  url = 	 {https://proceedings.mlr.press/v130/xing21c.html}
}

@InProceedings{Scetbon2023,
  title = 	 {Robust linear regression: {G}radient-descent, early-stopping, and beyond},
  author =       {Scetbon, Meyer and Dohmatob, Elvis},
  booktitle = 	 {Proceedings of The 26th International Conference on Artificial Intelligence and Statistics},
  pages = 	 {11583--11607},
  year = 	 {2023},
  url = 	 {https://proceedings.mlr.press/v206/scetbon23a.html}
}

@article{Hao2024surprising,
  title={The surprising harmfulness of benign overfitting for adversarial robustness},
  author={Hao, Yifan and Zhang, Tong},
  journal={arXiv preprint arXiv:2401.12236},
  year={2024}
}

@book{johnstone2017,
  title={{Gaussian Estimation: Sequence and Wavelet Models}},
  author={Johnstone, Iain M},
  publisher={Unpublished manuscript},
  year={2017}
}

@article{Xing2022Benefit,
author = {Xing, Yue and Song, Qifan and Cheng, Guang},
title = {Benefit of interpolation in nearest neighbor algorithms},
journal = {SIAM Journal on Mathematics of Data Science},
volume = {4},
number = {2},
pages = {935--956},
year = {2022},
doi = {10.1137/21M1437457},

URL = {

        https://doi.org/10.1137/21M1437457



},
eprint = {

        https://doi.org/10.1137/21M1437457



}
}

@article{Hassani2024,
author = {Hamed Hassani and Adel Javanmard},
title = {The curse of overparametrization in adversarial training: {P}recise analysis of robust generalization for random features regression},
volume = {52},
journal = {The Annals of Statistics},
number = {2},
publisher = {Institute of Mathematical Statistics},
pages = {441--465},
year = {2024},
doi = {10.1214/24-AOS2353},
URL = {https://doi.org/10.1214/24-AOS2353}
}

\end{sloppypar}
\end{document}